\documentclass{article}

\usepackage{microtype}
\usepackage{graphicx}
\usepackage{subcaption}
\usepackage{booktabs}
\usepackage{multirow}
\usepackage{xcolor}
\usepackage{hyperref}

\usepackage[preprint]{icml2026}
\AtBeginDocument{\hypersetup{pdfsubject={arXiv preprint}}}

\usepackage{amsmath}
\usepackage{amssymb}
\usepackage{mathtools}
\usepackage{amsthm}
\usepackage{bm}
\usepackage{enumitem}
\usepackage{float}
\usepackage{tikz}
\usetikzlibrary{arrows.meta,positioning,calc,shapes.misc}
\usepackage[capitalize,noabbrev]{cleveref}

\theoremstyle{plain}

\newtheorem{proposition}{Proposition}
\newtheorem{corollary}{Corollary}
\theoremstyle{definition}
\newtheorem{definition}{Definition}

\theoremstyle{remark}

\newcommand{\CG}{\mathrm{CG}}
\newcommand{\E}{\mathbb{E}}
\newcommand{\Var}{\operatorname{Var}}
\newcommand{\Cov}{\operatorname{Cov}}
\newcommand{\dob}{\operatorname{do}}
\newcommand{\bone}{b_{1}}
\newcommand{\bint}{\hat b_{\mathrm{int}}}
\newcommand{\bobs}{\hat b_{\mathrm{obs}}}
\newcommand{\IPT}{IPT}
\newcommand{\Vz}{\ensuremath{\mathrm{V0}}}
\newcommand{\Vo}{\ensuremath{\mathrm{V1}}}
\newcommand{\Vt}{\ensuremath{\mathrm{V2}}}
\definecolor{hl}{RGB}{43,108,176}
\definecolor{neg}{RGB}{197,48,48}

\icmltitlerunning{Forecasting Associations, Not Interventions}

\begin{document}
\raggedbottom

\twocolumn[
\icmltitle{Forecasting Associations, Not Interventions: A Preregistered Causal Audit\\ of Covariate-Aware Time Series Foundation Models}

\begin{icmlauthorlist}
\icmlauthor{Zong-Han Bai}{nthu}
\end{icmlauthorlist}
\icmlaffiliation{nthu}{National Tsing Hua University, Hsinchu, Taiwan}
\icmlcorrespondingauthor{Zong-Han Bai}{hu110048138@gapp.nthu.edu.tw}
\icmlkeywords{time series foundation models, causal inference, interventions, confounding, in-context learning, preregistration}
\vskip 0.3in
]

\printAffiliationsAndNotice{}

\begin{abstract}
Covariate-aware time series foundation models (TSFMs) accept a future path for a control variable and return a forecast, which invites \emph{what-if} queries: how would the outcome change if the control were set differently? The query is interventional; the models are trained on observational data. We determine whether their answers are the conditional slope $\partial\,\E[Y\mid X]$ or the interventional effect $\partial\,\E[Y\mid \dob(X)]$, at three levels. \emph{Theory}: a Bayes-optimal observational predictor answers interventional queries associationally unless the adjustment set is both available and used; for in-context least squares we derive the resulting \emph{confounding gap}---the difference between a model's effect estimate under an observational and under a randomized context---in closed form from context moments alone, giving a parameter-free reference line. \emph{Behaviour}: on a preregistered grid with exact counterfactual ground truth, Chronos-2, TiRex-2, TabPFN-TS and TimesFM-2.5 all track this line ($R^2$ 0.90--1.00) and their immediate responses are \emph{more} confounded than regression itself (slopes 1.3--1.8), on a nonlinear structural model as on the linear one. Pairs of structural models with identical observational laws prove this is not an identifiability limit: when the effect becomes identifiable, least squares recovers 100\% of the true effect difference, Chronos-2 75\%, TiRex-2 34\%, and the models separate into reproducible failure modes. \emph{Representation}: linear probes decode the data-generating regime (94--96\%) and the true effect ($R^2$ 0.7--0.9) from Chronos-2 and TiRex-2 embeddings that their forecasts ignore. On two physical systems with randomized actuators, pretrained models under-respond to interventions by 2--5$\times$ and ignore in-window confounding. Minutes of interventional post-training on 3{,}000 synthetic series re-map the latent information, cutting the gap by up to 83\% and transferring outside the training prior and to the nonlinear model. We show that the forecasting cost of doing so is a replay-distribution effect: replaying series from the model's own pretraining prior cuts it from $+47\%$ to $+10\%$ MASE, and with low-rank adaptation yields a recipe that meets the preregistered success criterion over three seeds---a 52--61\% gap reduction at a $+1$--$3\%$ forecasting cost---while improving the external DoTime benchmark. The preregistered protocol, the reference line, the failure-mode classifier and the recipe will be released.
\end{abstract}

% =====================================================================
\section{Introduction}
\label{sec:intro}
\begin{figure*}[t]
\centering
\includegraphics[width=\textwidth]{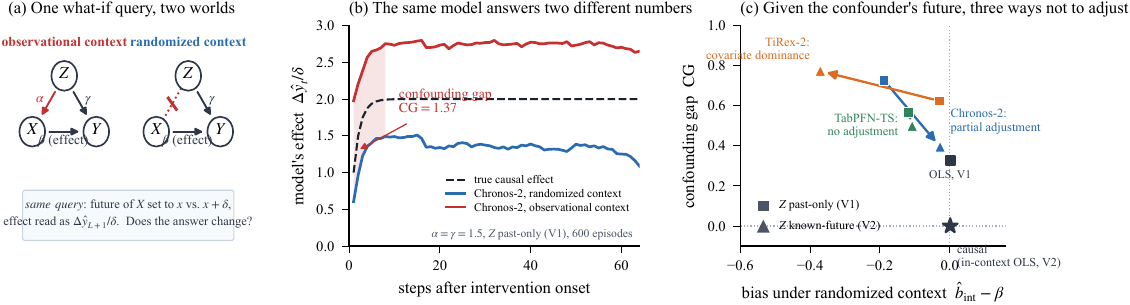}
\caption{\textbf{Foundation models forecast associations, not interventions.} (a)~One what-if query is posed under two contexts from the same structural model: observational ($Z$ drives both $X$ and $Y$) and randomized (the edge is cut). A model that estimates the effect answers the same number in both. (b)~Chronos-2 does not: under observational context its response to $\dob(X{+}\delta)$ is inflated by the confounding. The difference at the first intervened step is the \emph{confounding gap}, derived in closed form for in-context least squares (\cref{prop:cg}). (c)~Supplying the confounder's future should close the gap (star: least squares does); three models fail in three ways.}
\label{fig:overview}
\end{figure*}

The second generation of time series foundation models (TSFMs) is \emph{covariate-aware}: Chronos-2 \citep{ansari2025chronos2}, TiRex-2 \citep{tirex2_2026}, TabPFN-TS \citep{hoo2025tables} and TimesFM-2.5 \citep{das2024decoder} accept covariates alongside the target and condition the forecast on a user-supplied future path of those covariates. This is why practitioners reach for them in planning: supply the price schedule, the dosing plan or the actuator setting one intends to apply, and compare the forecast with the forecast under a different plan. The difference is read as the \emph{effect} of the plan. That reading is an interventional query---level two of Pearl's hierarchy \citep{pearl2009causality,bareinboim2022hierarchy}---posed to a model whose training objective is observational next-value prediction. Whether the answer is the effect of $\dob(X)$ or the conditional association between $X$ and $Y$ decides whether the plan will work. This paper settles the question for the four models above and gives the tools to settle it for the next one.

The closest prior work asks whether a model \emph{responds correctly}: \citet{won2026attenuated} shows that covariate-aware TSFMs attenuate and distort their dynamic response to counterfactual inputs on forced engineering systems, and that a short synthetic fine-tune restores the response magnitude. We ask whether a model responds \emph{for the right reason}. A model can respond with the correct magnitude and still be reading an association, and magnitude alone cannot tell the two apart; the failure we isolate is in fact the opposite of attenuation (\cref{sec:line}). The question cannot be settled on real data alone: on observational data the two answers differ only through confounding, and on experimental data they coincide. Our instrument holds the \emph{query} fixed and varies the \emph{regime} that generated the context. The \textbf{confounding gap} is the difference between a model's effect estimate when the context comes from an observational regime (a common cause $Z$ drives both the control $X$ and the outcome $Y$) and when it comes from a randomized regime ($X$ exogenous), everything else equal. A model that identifies the effect answers the same number in both regimes; a model that reads the association answers the observational slope in the first and the true effect in the second. The gap therefore measures, without ground truth, how much of a what-if answer is association. For in-context least squares we derive it in closed form from context moments: a parameter-free \emph{reference line} on which any model can be placed.

Against this reference we ask three questions. \emph{Identification}: is the effect identifiable from what the model is given? We vary the visibility of the confounder (hidden, past-only, known-future) and construct pairs of structural models with identical observational laws, so that ``the model answers the same number'' is correct in the non-identifiable case and a failure in the identifiable one. \emph{Utilisation}: when it is identifiable, does the model use the identifying information? Four TSFMs with four architectures do not: all lie on or above the reference line, on a linear and on a nonlinear structural model, and split into distinct failure modes when the confounder's future is supplied, while linear probes show that regime and effect are decodable from the encoders---present and unused. \emph{Learning}: what supervision changes the answer, at what cost, and does it transfer? Interventional post-training on a few thousand synthetic series removes most of the gap, inside and outside the training prior and on the nonlinear model. Its forecasting cost is a replay-distribution effect, not a price of causal supervision: with replay from the pretraining prior and low-rank adaptation, the recipe meets the preregistered success criterion.

\paragraph{Contributions.}
(i)~A theory-anchored audit: the confounding gap, a Bayes-optimality argument for why observational pretraining yields associational what-if answers (\cref{prop:bayes}), a closed-form reference line for in-context least squares under three confounder visibilities (\cref{prop:cg}), verified to $R^2=1.000$ by an explicit least-squares control, and a second reference (\cref{prop:static}) that predicts the magnitude of the models' excess confounding.
(ii)~Preregistered evidence that four covariate-aware TSFMs are associational---in their immediate response \emph{super}-associational---on linear and nonlinear structural models, a reproducible classifier of failure modes, and an identifiability split that separates ``cannot know'' from ``does not use''.
(iii)~A representational account: regime and effect are linearly decodable from Chronos-2 and TiRex-2 embeddings, and post-training changes the encoder, not only the readout.
(iv)~The first confounding audit of TSFMs on real physical systems with randomized actuators and an identified causal effect, and a complete repair trade-off (nineteen post-training variants, three inference-time repairs, two architectures) under a frozen preregistration, ending in a recipe that meets its success criterion and in the diagnosis that makes it work: the forecasting cost of interventional post-training is a replay-distribution effect, removed by replaying the pretraining prior.

% =====================================================================
\section{Related Work}
\label{sec:related}

\paragraph{Covariate-aware TSFMs and their audits.}
Chronos-2 \citep{ansari2025chronos2} extends the Chronos family \citep{ansari2024chronos} with group attention over the target and its covariates and distinguishes past-only from known-future covariates; TiRex-2 \citep{tirex2_2026} does the same in an xLSTM backbone \citep{beck2024xlstm}; TabPFN-TS \citep{hoo2025tables} casts forecasting as tabular regression with a prior-fitted network \citep{muller2022transformers} and uses known-future covariates only; TimesFM-2.5 \citep{das2024decoder} is univariate and attaches covariates through XReg, an in-context linear regression whose residual the model forecasts. None is trained with interventional data or documents what a difference of covariate-conditioned forecasts estimates. \citet{won2026attenuated} is the closest work: it measures the response of Chronos-2, TimesFM-2.5 and TabPFN-TS to counterfactual inputs on forced synthetic and four measured engineering systems against classical system identification, finds attenuated and misshaped responses, and repairs them with context dither and a 26-minute fine-tune on 3{,}000 synthetic forced-system series, at a loss of univariate forecasting skill that mixing in real M4 series did not prevent. Its question is system identification: is the dynamic response right? Ours is causal: is the response an effect or an association? The two are complementary, and \cref{sec:repair} explains Won's forecasting cost and shows how to remove it. \citet{jander2026causal} intervene on the data generator to probe sensitivity, whereas we ask whether the answer to an interventional \emph{query} is causal. Leakage concerns \citep{leakage2025rethinking} motivate synthetic data with exact counterfactuals for our headline claims.

\paragraph{Causal benchmarks and causal foundation models.}
DoTime \citep{thumm2026dotime} and CausalTimePrior \citep{thumm2026priors} provide synthetic interventional benchmarks and train causal prior-fitted models from scratch; we use the frozen DoTime suite as an external test and show that an off-the-shelf forecaster can be post-trained in minutes rather than replaced. Causal Chambers \citep{gamella2025causal} supply physical systems with randomized actuators, which make the causal effect identifiable and let the confounding audit, not only the response, be run on real data. The closest conceptual neighbour is the observation that LLM-simulated experiments are observational studies in which ``intervening'' on a prompt drags latent attributes along \citep{illusion2026intervention}; our abduction drift (\cref{prop:cg}) is the time-series analogue \citep[see also][]{causalarena2026,kaddour2022causal,scholkopf2021toward}.

Transformers trained on function classes implement (ridge-regularised) least squares in context \citep{garg2022what,akyurek2023what,vonoswald2023transformers}, the hypothesis behind our reference line; linear probes \citep{alain2017understanding}, replay and low-rank adaptation \citep{kirkpatrick2017overcoming,rolnick2019experience,hu2022lora} enter \cref{sec:probe,sec:repair}.

% =====================================================================
\section{Framework}
\label{sec:framework}

\subsection{Setting}
A covariate-aware forecaster is a map $f(c,x_{\mathrm{fut}})$ from a context $c=(y_{1:L},x_{1:L},z_{1:L})$ and a future path $x_{\mathrm{fut}}=x_{L+1:L+H}$ of the control to a forecast $\hat y_{L+1:L+H}$; known-future covariates are supplied over the horizon, past-only covariates appear in $c$ only. A \emph{what-if query} supplies $x_{\mathrm{fut}}$ and $x_{\mathrm{fut}}+\delta$ and reads the forecast difference as the effect of setting $X$ to $x_{\mathrm{fut}}+\delta$. Our ground-truth generator is the lagged linear-Gaussian structural model
\begin{equation}
\label{eq:scm}
\begin{aligned}
Z_t&=\rho_z Z_{t-1}+\varepsilon^z_t,\qquad X_t=\rho_x X_{t-1}+\alpha Z_{t-1}+\varepsilon^x_t,\\
Y_t&=\rho_y Y_{t-1}+\beta X_t+\gamma Z_{t-1}+\varepsilon^y_t,
\end{aligned}
\end{equation}
with independent Gaussian noise, so that $Z$ is a lagged common cause of $X$ and $Y$ (back-door path $X\leftarrow Z\rightarrow Y$) and the interventional effect of a unit shift of $X_t$ on $Y_t$ is $\beta$. The \emph{observational regime} has $\alpha\neq0$; the \emph{randomized regime} sets $\alpha=0$, which is what an experiment would produce. Under $\dob(X_{\mathrm{fut}}\leftarrow x_{\mathrm{fut}}+\delta)$ the future of $Z$ is unaffected because $Z$ is a cause, not an effect, of $X$; simulating both arms with shared noise gives the exact counterfactual effect path $\Delta y$ \citep{pearl2009causality}.

\paragraph{Confounder visibility.}
\Vz: $Z$ absent from $c$. \Vo: $Z$ in $c$ as a past-only covariate (the common deployment setting). \Vt: $Z$ in $c$ and over the horizon as a known-future covariate, so that the back-door adjustment set is available at query time. \Vz\ is not identifying and serves as a check on the theory.

\subsection{Effect readouts}
Let $\Delta\hat y=f(c,x_{\mathrm{fut}}+\delta)-f(c,x_{\mathrm{fut}})$. We report the \textbf{impact coefficient} $\bone=\Delta\hat y_{L+1}/\delta$ (first intervened step; the quantity \cref{prop:cg} predicts) and the \textbf{cumulative coefficient} $\hat b=\langle\Delta\hat y,\Delta y\rangle/\langle\Delta y,\Delta y\rangle$, the projection of the model's effect path on the true counterfactual path. Sensitivity $S=\lVert\Delta\hat y\rVert/\lVert\Delta y\rVert$ and direction accuracy are reported for comparability with prior work \citep{won2026attenuated,thumm2026dotime}.

\subsection{The confounding gap}
\begin{definition}[Confounding gap]
\label{def:cg}
For a query $\dob(\delta)$ and a model $f$, $\CG(f)=\E_{\mathrm{obs}}[\bone]-\E_{\mathrm{int}}[\bone]$: the difference between the model's effect estimate when the context is drawn from the observational regime and when it is drawn from the randomized regime, all else (structural model, query, visibility) equal.
\end{definition}
A model that identifies $\beta$ from $c$ has $\CG=0$; a model that reads the observational association answers $\beta+\text{bias}$ and $\beta$ respectively, so $\CG$ equals the bias: a ground-truth-free measurement of how much of the answer is association. Two properties fix its use. $\CG=0$ is necessary, not sufficient: a model that ignores $X$ altogether also has $\CG=0$, so we pair the gap with the \textbf{randomized-context bias} $\bint-\beta$ and the sensitivity $S$; the pair $(\bint-\beta,\ \CG)$ is the \emph{causal compass} of \cref{fig:overview}c, on which a causal predictor sits at the origin, an associational one on the vertical axis, and an unresponsive one at $(-\beta,0)$. And the gap needs two comparable contexts, one confounded and one not: generated on synthetic data, obtained from a regime variable on real data (\cref{sec:real}).

\subsection{What an observational predictor answers}
\begin{proposition}[Observational predictors answer associationally]
\label{prop:bayes}
Let $f$ be Bayes-optimal for an observational pretraining distribution, $f(c,x_{\mathrm{fut}})=\E[Y_{\mathrm{fut}}\mid c,X_{\mathrm{fut}}=x_{\mathrm{fut}}]$. Then
\begin{align*}
f(c,x+\delta)-f(c,x)={}&\E[Y_{\mathrm{fut}}\mid c,\dob(x+\delta)]\\&-\E[Y_{\mathrm{fut}}\mid c,\dob(x)]+\mathcal{B}(c,x,\delta),
\end{align*}
where the first two terms form the interventional contrast and the confounding term $\mathcal{B}$ vanishes whenever the variables in $c$ that $f$ conditions on block every back-door path from $X_{\mathrm{fut}}$ to $Y_{\mathrm{fut}}$ (Pearl's back-door criterion relative to $c$). In the model \eqref{eq:scm}, $\mathcal{B}\neq0$ under \Vz\ and \Vo\ whenever $\alpha\gamma\neq0$, and equals the expressions of \cref{prop:cg}.
\end{proposition}
\Cref{prop:bayes} (proof in \cref{app:proofs}) is architecture-free: the \emph{objective} used to pretrain TSFMs makes associational what-if answers the default, and the only exits are to change the training distribution (\cref{sec:repair}) or to supply the adjustment set and have the model use it (\Vt). It is also model-free in the other direction: for \emph{any} structural model in which the horizon values of the supplied covariates block every back-door path, a predictor that conditions on them has $\CG=0$, so the \Vt\ prediction of zero gap holds beyond the linear family (\cref{cor:general}). \Cref{prop:cg} supplies the reference against which to measure which exit a model takes.

\begin{proposition}[Confounding gap of in-context least squares]
\label{prop:cg}
Let $f$ fit $Y_t$ by least squares on the regressors visible in $c$ and apply the fitted coefficients to the horizon. As $L\to\infty$:
\begin{itemize}[leftmargin=1.2em,itemsep=1pt,topsep=1pt]
\item[\Vz] $\bone\to\beta+\gamma\,\pi_X(\text{regime})$, where $\pi_X$ is the coefficient of $X_t$ in the population regression of the omitted $Z_{t-1}$ on $(1,Y_{t-1},X_t)$ (Frisch--Waugh--Lovell). $\CG=\gamma[\pi_X(\mathrm{obs})-\pi_X(\mathrm{int})]$; note $\pi_X(\mathrm{int})\neq0$ because $Z_{t-1}$ remains correlated with the lagged regressor $Y_{t-1}$.
\item[\Vo] the coefficient on $X$ is $\beta$ in both regimes, but the unknown future of $Z$ is imputed through the in-context association $g(\text{regime})=\Cov(Z_{t-1},X_t)/\Var(X_t)$, so the effective response is $\beta+\gamma g$ with $g(\mathrm{int})=0$: $\CG=\gamma g(\mathrm{obs})$. This is the \emph{abduction drift}: intervening on $X$ drags the model's belief about $Z$.
\item[\Vt] with $Z$'s future supplied and unchanged under $\dob(\delta)$, $\bone=\beta$ in both regimes and $\CG=0$.
\end{itemize}
All three quantities are functions of context sample moments alone.
\end{proposition}
An explicit least-squares forecaster reproduces \cref{prop:cg} with slope $1.000$ and $R^2=1.000$ across visibilities (\cref{fig:refline}, right), validating both the proposition and the measurement harness before any foundation model is run.

\begin{proposition}[Shrinkage]
\label{prop:shrink}
If the in-context regression is ridge-regularised with penalty $\lambda$ \citep{vonoswald2023transformers}, the coefficient on the partialled-out control is multiplied by $v/(v+\lambda/L)$, $v$ its partial variance: under randomized context $S<1$, increasing in $L$ and $v$; under observational context the confounding term of \cref{prop:cg} does not vanish with $L$.
\end{proposition}

\begin{proposition}[Static-regression reference]
\label{prop:static}
Let $f_{\mathrm{s}}$ regress $Y_t$ on $(1,X_t)$ (\Vz) or $(1,X_t,Z_{t-1})$ (\Vo, \Vt) \emph{without} the autoregressive term and apply the slope to the horizon. In \eqref{eq:scm}, under \Vz,
\[
b_{\mathrm{s}}=\beta+\gamma\,\frac{\Cov(Z_{t-1},X_t)}{\Var(X_t)}+\rho_y\,\frac{\Cov(Y_{t-1},X_t)}{\Var(X_t)},
\]
so the static gap $\CG_{\mathrm{s}}=b_{\mathrm{s}}(\mathrm{obs})-b_{\mathrm{s}}(\mathrm{int})$ exceeds the dynamic gap of \cref{prop:cg} by the autoregressively accumulated confounding; on our grid $\CG_{\mathrm{s}}/\CG=2.1$ (\Vz) and $2.6$ (\Vo). A predictor that answers a what-if query with the static association therefore has impact slope above 1 on the \cref{prop:cg} line.
\end{proposition}
\Cref{prop:static} turns super-association from a label into a hypothesis with a predicted value: the TSFMs answer with the long-run association that a lag-free regression would report, not with the contemporaneous coefficient and its dynamics. \Cref{sec:line} tests it.

\paragraph{Preregistration.}
\label{sec:questions}
All classification thresholds were frozen before any foundation model was run (\texttt{preregister.yaml}, tag \texttt{prereg-freeze}); three later changes---the impact-coefficient definition, an exploratory label for slopes above the preregistered range, and a bug fix in the forecasting-cost evaluation---are recorded as amendments and reported as such (\cref{app:prereg}), following \citet{nosek2018preregistration}.

% =====================================================================
\section{Experimental Protocol}
\label{sec:protocol}
\textbf{Grid.} $\alpha,\gamma\in\{0,0.25,0.5,1,1.5\}$, $\beta=1$, $\rho=(0.6,0.5,0.5)$, unit noise, $L=512$, $H=64$, $\delta=1\sigma_x$; 200 episodes per cell and generator seed, three seeds (600 per cell), for three visibilities and two regimes. Slopes and $R^2$ are computed over the 16 confounded cells ($\alpha\gamma>0$); \Vt\ is summarised as the residual gap relative to \Vo\ on matched cells, since its prediction is identically zero. Bootstrap CIs use 1{,}000 resamples \citep{efron1979bootstrap}.
\textbf{Models.} Chronos-2 and TiRex-2 run the full grid; TabPFN-TS (1.6\,s per forecast) runs a $3\times3$ sub-grid with one seed and $n=100$, and only \Vz/\Vt\ because it drops past-only covariates by construction; TimesFM-2.5 runs the $3\times3$ sub-grid with one seed, both univariate and with the XReg covariate layer. An in-context least-squares forecaster implements \cref{prop:cg} as the theory control; the median forecast is used throughout.
\textbf{Experiments.} E1 the gap on the grid; E2 null effect ($\beta=0$); E3 context length $L\in\{128,512,2048\}$, dose $\delta\in\{0.5,1,4\}\sigma$, anticipation; E4 the frozen DoTime \texttt{dot-Identifiability-v1} suite (12{,}150 episodes, eight structures, official protocol) \citep{thumm2026dotime}; E5 repairs; E6 two Causal Chambers systems \citep{gamella2025causal}; E7 matched observational distributions; E8 out-of-prior transfer; E9 linear probes; E11 a nonlinear structural model with a nonparametric reference. Compute: one 12\,GB laptop GPU. Details in \cref{app:details}.

% =====================================================================
\section{Behaviour: Four Models, One Reference Line}
\label{sec:behaviour}

\begin{figure*}[t]
\centering
\includegraphics[width=\textwidth]{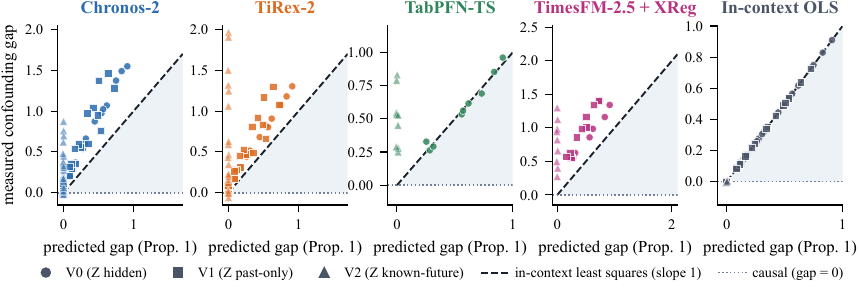}
\caption{\textbf{Measured confounding gap against the \cref{prop:cg} prediction}, one point per $(\alpha,\gamma)$ cell (slopes and $R^2$ in \cref{tab:e1}). The identity line is in-context least squares; the horizontal line is a causal predictor. All four TSFMs lie on or above the line with the confounder hidden (\Vz) or past-only (\Vo); supplying the confounder's future (\Vt, theory 0) removes only part of the gap and, for TiRex-2 and TimesFM-2.5, little or none of it. Right: the least-squares control validates the harness ($R^2=1.000$).}
\label{fig:refline}
\end{figure*}

\begin{table}[t]
\caption{\textbf{E1 confounding gap} ($L=512$; slope and $R^2$ of the measured impact-coefficient gap on the \cref{prop:cg} prediction over confounded cells, 95\% cell-bootstrap CIs in \cref{tab:ci}). $\CG_{\Vt}/\CG_{\Vo}$: residual gap when $Z$'s future is supplied (theory 0); $\bint$: randomized-context estimate of $\beta=1$. $^{\dagger}$relative to \Vz\ (TabPFN-TS drops past-only covariates).}
\label{tab:e1}
\vskip 0.08in
\centering
\footnotesize
\setlength{\tabcolsep}{2pt}
\begin{tabular}{@{}llcccc@{}}
\toprule
Model & Vis. & slope ($R^2$) & $\overline{\CG}$ & $\CG_{\Vt}/\CG_{\Vo}$ & $\bint$ \\
\midrule
\multirow{3}{*}{Chronos-2} & \Vz & 1.64 (0.99) & 0.77 & & 0.84 \\
 & \Vo & 1.78 (0.83) & 0.73 & & 0.81 \\
 & \Vt & -- & 0.40 & \textbf{0.58} & 0.97 \\
\midrule
\multirow{3}{*}{TiRex-2} & \Vz & 1.34 (1.00) & 0.62 & & 0.95 \\
 & \Vo & 1.56 (0.83) & 0.62 & & 0.97 \\
 & \Vt & -- & 0.77 & \textbf{1.07} & \textbf{0.63} \\
\midrule
\multirow{2}{*}{TabPFN-TS} & \Vz & 1.00 (0.98) & 0.57 & & 0.88 \\
 & \Vt & -- & 0.50 & \textbf{0.89}$^{\dagger}$ & 0.89 \\
\midrule
\multirow{3}{*}{TimesFM-2.5+XReg} & \Vz & 1.26 (0.90) & 0.95 & & 1.32 \\
 & \Vo & 1.62 (0.96) & 0.95 & & 1.32 \\
 & \Vt & -- & 0.78 & \textbf{0.88} & 1.32 \\
\midrule
TimesFM-2.5 (univ.) & all & -- & 0.00 & -- & 0.00 \\
\bottomrule
\end{tabular}
\end{table}

\subsection{All four models are associational; three are super-associational}
\label{sec:line}
\Cref{fig:refline} and \cref{tab:e1} give the headline result. With the confounder hidden, the measured gap of every covariate-aware model is a linear function of the \cref{prop:cg} prediction ($R^2$ 0.90--1.00), and no model lies below the line. TabPFN-TS lies exactly on it (slope 1.00): a prior-fitted tabular regressor is in-context least squares, as its design implies \citep{muller2022transformers}. Chronos-2, TiRex-2 and TimesFM-2.5 with XReg lie \emph{above} it: their impact coefficients are 1.3--1.8 times more confounded than regression (cell-bootstrap 95\% CIs exclude 1), while the cumulative coefficient has slope 0.6--1.2 (\cref{tab:e1full}). TimesFM-2.5 \emph{without} its covariate layer does not respond to the intervention at all ($S=0$, $\CG=0$): the preregistered pattern~D, and the live example of why a zero gap is not a causal answer. The gap path over the horizon explains both slopes (\cref{fig:step}): regression answers with the contemporaneous coefficient and accumulates, whereas the TSFMs answer with the \emph{long-run} association from the first step, so their gap is flat at 0.8--0.9 throughout. \Cref{prop:static} makes this quantitative. Regressing the same measured gaps on the static-regression reference instead of the dynamic one gives slopes of 1.17/1.07 (Chronos-2, \Vz/\Vo), 0.95/0.94 (TiRex-2) and 1.006 (TimesFM-2.5 with XReg, a static regression by construction) at $R^2$ 0.94--1.00, with residuals three times smaller than on the dynamic line (\cref{tab:static}); TabPFN-TS alone stays on the dynamic line. The models answer the what-if query with the static association: \emph{super-association} is a measurable departure from the Bayes-optimal reference toward \emph{more} confounding, with a predicted magnitude. It is also the opposite of the attenuation reported by \citet{won2026attenuated} on forced systems, and the two coexist: under randomized context the same models attenuate ($\bint$ 0.81--0.97, $S<1$), under confounded context they over-respond. Response magnitude alone cannot diagnose which is happening; the gap can. It falls outside the preregistered range for pattern A ($[0.6,1.4]$), so the frozen verdict is \emph{unclassified}; we report it as such and attach the exploratory label A$^+$ (amendment A2). Two further facts matter for deployment. \Vo$\approx$\Vz: the confounder's past changes the level forecast but not the effect estimate (Chronos-2 mean gap 0.73 vs 0.77; TiRex-2 0.62 vs 0.62): past-only covariates are read for prediction and ignored for adjustment. And the randomized-context estimates are close to the truth ($\bint$ 0.81--0.97, mild attenuation; \cref{prop:shrink}): the failure is specific to confounded context, which is the context users have.

\subsection{Three failure modes when the adjustment set is supplied}
\label{sec:modes}
\begin{figure}[t]
\centering
\includegraphics[width=\columnwidth]{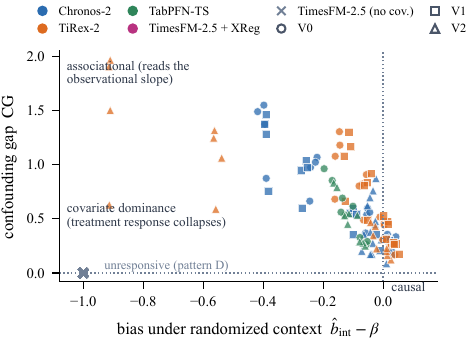}
\caption{\textbf{Causal compass.} Each point is a confounded cell. A causal predictor sits at the origin, an associational one on the vertical axis; covariate dominance pushes left. Chronos-2 (\Vt) moves toward the origin but stops; TabPFN-TS does not move; TiRex-2 (\Vt) moves \emph{left}: supplying $Z$'s future collapses its response to $X$.}
\label{fig:compass}
\end{figure}
\Vt\ is the setting in which a what-if query is answerable: the back-door set is in the context and unchanged by the intervention, and least squares removes the gap entirely. The TSFMs do not, and they fail in three distinct, reproducible ways (\cref{fig:compass}; classifier with fixed thresholds in \cref{app:classifier}).
\emph{Partial adjustment} (Chronos-2): the gap shrinks to 58\% of its \Vo\ value and $\bint$ rises to 0.97; the known-future confounder is used, but only partly.
\emph{No adjustment} (TabPFN-TS): the gap is 89\% of its \Vz\ value; the covariate is entered as a regressor but the what-if contrast is read off the marginal association.
\emph{Covariate dominance} (TiRex-2): the gap does not shrink at all (107\%), but for the opposite reason---the randomized-context estimate collapses from 0.97 to 0.63 (to 0.09 at $\gamma=1.5$, and even at $\alpha=0$ where there is no confounding): a strong known-future covariate crowds out the control, a second failure axis that the compass separates from confounding. TimesFM-2.5 with XReg joins the \emph{no adjustment} mode (88\% of the gap remains): its linear covariate layer has no autoregressive term, so entering $Z$ removes only the contemporaneous part of the confounding.

\begin{figure*}[t]
\begin{minipage}[t]{0.49\textwidth}
\centering
\includegraphics[width=\linewidth]{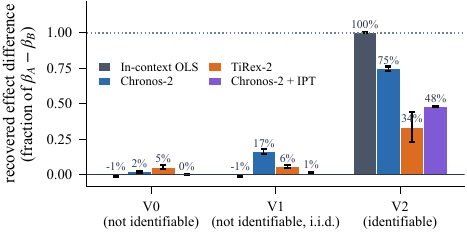}
\caption{\textbf{Matched observational distributions (E7).} Pairs of structural models with identical $(X,Y)$ law but effects $\beta_A=1$ and $\beta_B\in\{0.5,0.25,0\}$ (mean $\pm$ s.d. over pairs). Where the effect is not identifiable every model, including least squares, answers the same number: not a failure. Where it is (\Vt), least squares recovers the full difference, Chronos-2 three quarters, TiRex-2 one third: information present, not used.}
\label{fig:ident}
\end{minipage}\hfill
\begin{minipage}[t]{0.49\textwidth}
\centering
\includegraphics[width=\linewidth]{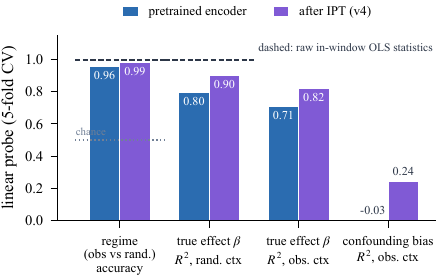}
\caption{\textbf{Linear probes on Chronos-2 encoder embeddings} ([REG] token $+$ patch mean; 600 episodes; held-out 5-fold). The pretrained encoder separates randomized from confounded contexts (96\%) and encodes the true effect ($R^2$ 0.80/0.71) but not the confounding bias ($R^2\approx0$). After post-training all four rise: the encoder begins to represent the bias it must subtract.}
\label{fig:probe}
\end{minipage}
\end{figure*}
\subsection{Effects from nothing, and not a finite-sample effect}
With $\beta=0$ there is no causal edge from $X$ to $Y$, yet under observational context every model reports one (\cref{tab:null}): 0.47 (Chronos-2), 0.48 (TiRex-2), 0.44 (TabPFN-TS) and 0.72 (TimesFM-2.5+XReg) against a least-squares omitted-variable bias of 0.42, with a 100\% false-positive rate across cells; under randomized context the same models report 0.00. Supplying $Z$'s future removes two thirds of it for Chronos-2, one third for TiRex-2, almost none for TabPFN-TS.
The gap is not a small-sample artefact of in-context estimation; if it were, it would shrink with $L$ (\cref{fig:shrink}, \cref{app:extra}). Under randomized context sensitivity rises with $L$ (Chronos-2 \Vt: 0.77$\to$0.96$\to$0.99 for $L=128,512,2048$), the ridge-shrinkage signature of \cref{prop:shrink}; under observational context it stays at 1.3--1.4 at every $L$ (at $L=2048$ the Chronos-2 gap is still 0.98/0.85/0.39 for \Vz/\Vo/\Vt), and dose behaves likewise ($\bint$ 0.72, 0.77, 0.94 for $\delta=0.5,1,4\sigma$, gap unchanged). TiRex-2 under \Vt\ is again the exception: its sensitivity \emph{falls} with $L$ (0.65$\to$0.56$\to$0.52). Anticipation before the onset is absent (median leak ratio $\le0.035$).

\subsection{The identifiability split}
\label{sec:ident}
The natural objection is that observational data are not identifying, so an associational answer is the only possible one. That is true in \Vz\ and false in \Vt, and E7 makes the distinction operational. We construct pairs of i.i.d.\ structural models whose $(X,Y)$ joint laws coincide to the second moment but whose effects differ ($\beta_A=1$ vs $\beta_B$), and ask each model for its effect on each member (\cref{fig:ident}). Under \Vz, and under \Vo\ in this design (without temporal precedence the data cannot tell $Z\to X$ from $X\to Z$, and the abduction drift is equal for both members by construction), no estimator can separate the pair: least squares recovers $-1\%$ of the difference, Chronos-2 2--17\%, TiRex-2 5--6\%; agreement is the correct answer. Under \Vt\ the effect is identifiable and least squares recovers 100\%; Chronos-2 recovers 75\%, TiRex-2 34\%. This isolates the utilisation failure from identification: the information is in the context, and the models leave most of it on the table.

\subsection{A nonlinear structural model}
\label{sec:nonlin}
\begin{figure*}[t]
\centering
\includegraphics[width=\textwidth]{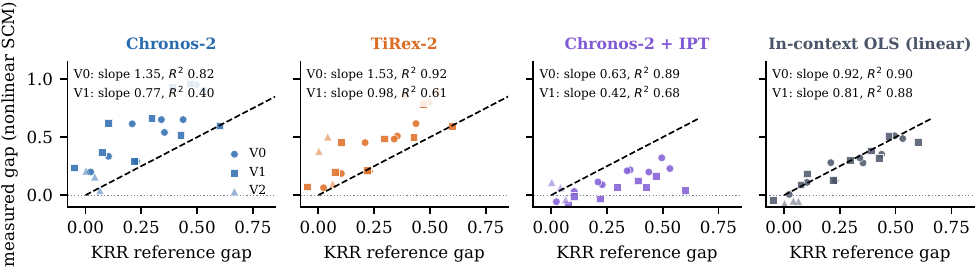}
\caption{\textbf{Nonlinear structural model (E11).} $X$ responds to $\tanh(Z)$, $Y$ to $1.5\tanh(X)$ and a convex function of $Z$; the reference is an in-context kernel ridge regression with the same three visibility rules, and its own observational-minus-randomized gap is the $x$-axis. Chronos-2 and TiRex-2 again sit above the reference under \Vz\ (slopes 1.35 and 1.53, $R^2$ 0.82 and 0.92); post-training on the \emph{linear} family transfers (gap 0.57$\to$0.15 under \Vz, 0.58$\to$0.02 under \Vo); the linear least-squares control is itself close to the nonparametric reference (slope 0.92, $R^2$ 0.90).}
\label{fig:nonlin}
\end{figure*}
The closed form of \cref{prop:cg} is linear-Gaussian; the behaviour is not. E11 replaces the structural equations by saturating and convex nonlinearities and the reference by an in-context kernel ridge regression (\cref{app:details}), whose own gap under each regime is the $x$-axis of \cref{fig:nonlin}. The picture is unchanged: Chronos-2 and TiRex-2 lie above the nonparametric reference under \Vz\ (slopes 1.35 and 1.53, $R^2$ 0.82 and 0.92; mean gaps 0.57 and 0.47 against 0.30 for the reference), TiRex-2 keeps its covariate-dominance signature under \Vt\ (gap 0.55, $\bint$ 0.36), and Chronos-2 adjusts more than on the linear model when $Z$'s future is supplied (16\% of its \Vo\ gap remains). The linear least-squares control stays close to the nonparametric one (slope 0.92, $R^2$ 0.90), which is why the linear reference line is informative beyond the linear family.

\subsection{An external nonlinear benchmark}
On the frozen DoTime identifiability suite (12{,}150 episodes, nonlinear generators, eight structures) Chronos-2 and TiRex-2 attain the best level accuracy of any non-oracle method (NMSE 0.44 and 0.49), yet their effect-direction accuracy (0.68, 0.67) is \emph{below} that of a naive unadjusted regression (0.77), and on the back-door structure they score 0.68 against 0.81 for the adjusting estimator (\cref{tab:e4}). Strong forecasting and absent adjustment coexist outside our generator family.

% =====================================================================
\section{Representation: The Information Is There}
\label{sec:probe}
\begin{figure*}[t]
\centering
\includegraphics[width=\textwidth]{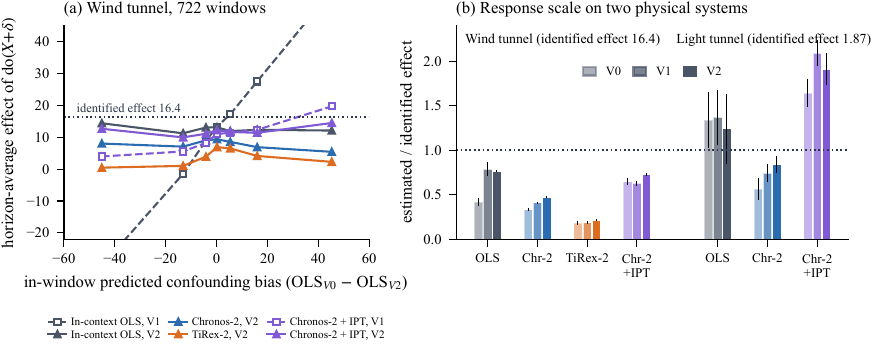}
\caption{\textbf{Causal Chambers.} (a)~Wind tunnel: horizon-average effect of $\dob(\text{load\_in}+\delta)$ on downwind pressure versus the in-window confounding bias predicted from context (septile bins). An associational predictor has slope 1 (least squares, \Vo), an adjusting one slope 0 (least squares, \Vt). Pretrained models have slope $\approx0$ at every visibility \emph{and} a response one third to one half of the identified effect: they ignore the covariate rather than adjust for it. (b)~Response scale on both systems: on the light tunnel the post-trained model overshoots by $1.6$--$2.1\times$.}
\label{fig:real}
\end{figure*}
Why do three thousand synthetic series suffice to change a what-if answer (\cref{sec:repair})? Because the encoder already carries what adjustment needs; only the readout is wrong. We establish this with linear probes \citep{alain2017understanding} on frozen Chronos-2 embeddings of $(y,x,z)$ contexts from both regimes with $\beta\sim U[-1.5,1.5]$ (\cref{fig:probe}). The regime---whether $X$ was randomized---is linearly decodable at 96\% (chance 50\%); the true effect $\beta$ at $R^2=0.80$ under randomized and $0.71$ under observational context; the \cref{prop:cg} confounding bias is not ($R^2\approx0$). Raw in-window least-squares statistics reach 99.8\% and $R^2=0.997$, so the encoder \emph{preserves} this information rather than computing the estimator. The forecasts of \cref{sec:behaviour} use none of it: the model represents that the context is confounded and what the effect is, and answers the association anyway. After post-training, regime accuracy rises to 99\%, effect $R^2$ to 0.90/0.82, and the bias becomes decodable ($R^2=0.24$): the repair changes the representation, not only the output head. The same probes on TiRex-2's post-stack hidden states give the same picture across architectures: regime 94\%, effect $R^2$ 0.90/0.86, bias $R^2$ 0.15 rising to 0.24 after one round of post-training (\cref{tab:probe}).

% =====================================================================
\section{Real Physical Systems}
\label{sec:real}
\begin{table*}[t]
\caption{\textbf{Repairs} (Chronos-2 unless stated; $3\times3$ confounded grid; KS = KernelSynth). Gap reduction relative to the pretrained model; $\bint$ under \Vo; spurious effect at $\beta=0$ under observational \Vo\ context (pretrained 0.47; least squares 0.42); change in M4-hourly MASE (mean / median over 200 series); DoTime effect-direction accuracy (pretrained 0.684). Preregistered success ($\checkmark$): $\ge50\%$ reduction on both visibilities \emph{and} MASE change $\le5\%$ on both the mean and the median.}
\label{tab:repair}
\vskip 0.08in
\centering
\scriptsize
\setlength{\tabcolsep}{3pt}
\begin{tabular}{@{}llcccccc@{}}
\toprule
& Repair & $\CG$ red.\ \Vo & $\CG$ red.\ \Vt & $\bint$ (\Vo) & spurious & MASE mean / med. & DoTime \\
\midrule
\multirow{4}{*}{\rotatebox{90}{\scriptsize inference}}
& context dither ($0.5\sigma\times4$) & $-6\%$ & $-34\%$ & 0.78 & -- & 0 & -- \\
& in-context randomization, $K{=}2$ & $-19\%$ & $+5\%$ & 0.90 & -- & 0 & -- \\
& in-context randomization, $K{=}8$ & $-22\%$ & $+20\%$ & 0.88 & -- & 0 & -- \\
& do-flag without training & $-4\%$ & $+13\%$ & 0.74 & -- & 0 & -- \\
\midrule
\multirow{9}{*}{\rotatebox{90}{\scriptsize post-training}}
& v2: \Vo\ contexts, $\beta$ incl.\ 0 & $-85\%$ & $-44\%$ & 0.96 & 0.28 & $+71$ / $+21\%$ & 0.693 \\
& v3: \Vt\ contexts & $-70\%$ & $\mathbf{-97\%}$ & 0.64 & 0.21 & $+46$ / $+13\%$ & 0.697 \\
& \textbf{v4: mixed \Vo/\Vt} (3 seeds) & $\mathbf{-78\%}$ ($76\pm2$) & $\mathbf{-82\%}$ ($88\pm5$) & 0.86 & 0.25 & $+47$ / $+10\%$ & \textbf{0.709} \\
& v5: v4 $+$ signed $\gamma$ & $-85\%$ & $-93\%$ & 1.11 & 0.35 & $+50$ / $+11\%$ & -- \\
& v6: v4, full, lr $10^{-6}$ & $-36\%$ & $-39\%$ & 0.93 & -- & $+22$ / $+11\%$ & -- \\
& v7: v4, LoRA, lr $10^{-5}$ & $-14\%$ & $-7\%$ & 0.83 & -- & $\mathbf{+1}$ / $\mathbf{+3\%}$ & -- \\
& v8: v4, LoRA, lr $10^{-4}$ & $-57\%$ & $-62\%$ & 0.96 & -- & $+17$ / $+12\%$ & -- \\
& v9: v4 $+$ real-series replay & $-67\%$ & $-70\%$ & 1.01 & -- & $+32$ / $+23\%$ & -- \\
& v10: do-flag counterfactual supervision & $-55\%$ & $-74\%$ & 1.13 & 0.42 & $+31$ / $+10\%$ & -- \\
\midrule
\multirow{9}{*}{\rotatebox{90}{\scriptsize replay / regularisation}}
& \textbf{v11: v4 $+$ KS replay 3{,}000} & $\mathbf{-83\%}$ & $\mathbf{-82\%}$ & 1.05 & 0.32 & $\mathbf{+10}$ / $+11\%$ & -- \\
& v12: v4, 50\% observational-regime series & $-62\%$ & $-54\%$ & 0.89 & 0.30 & $+49$ / $+11\%$ & -- \\
& v13: LoRA $10^{-4}$ $+$ KS replay 3{,}000 & $-46\%$ & $-53\%$ & 0.93 & 0.43 & $\mathbf{+2}$ / $\mathbf{+3\%}$ & -- \\
& v14: v4 $+$ L2-SP ($\lambda=10^{-2}$) & $-82\%$ & $-84\%$ & 0.97 & -- & $+50$ / $+10\%$ & -- \\
& v15: LoRA $2{\times}10^{-4}$ $+$ KS replay 3{,}000 & $-71\%$ & $-61\%$ & 1.03 & -- & $+4$ / $+6\%$ & -- \\
& v16: v13, 3{,}000 steps & $-54\%$ & $-56\%$ & 1.01 & -- & $+3$ / $+7\%$ & -- \\
& v17: v13, KS replay 1{,}500 & $-43\%$ & $-49\%$ & 0.89 & -- & $\mathbf{0}$ / $\mathbf{+2\%}$ & -- \\
& \textbf{v18: LoRA $1.5{\times}10^{-4}$ $+$ KS replay 3{,}000} $\checkmark$ & $\mathbf{-59\%}$ & $\mathbf{-58\%}$ & 0.98 & -- & $\mathbf{+3}$ / $\mathbf{+4\%}$ & -- \\
& \textbf{v19: LoRA $2{\times}10^{-4}$ $+$ KS replay 4{,}500} $\checkmark$ (3 seeds) & $\mathbf{-54\%}$ ($52\pm1$) & $\mathbf{-63\%}$ ($61\pm6$) & 1.01 & 0.41 & $\mathbf{+1}$ / $\mathbf{+2\%}$ ($1.4{\pm}0.5$ / $2.9{\pm}1.3$) & \textbf{0.710} \\
& v20: v19 $+$ wide scale randomization & $-29\%$ & $-55\%$ & 0.95 & -- & $+1$ / $+2\%$ & -- \\
& v22: v19 $+$ $X,Y$ scale randomization $\checkmark$ & $-50\%$ & $-56\%$ & 0.94 & -- & $+3$ / $+3\%$ & -- \\
& v23: v19 $+$ scale rand.\ on half the series & $-50\%$ & $-55\%$ & 0.94 & -- & $+1$ / $+2\%$ & -- \\
& v21: v19 $+$ signed $\gamma$ & $-36\%$ & $-26\%$ & 0.96 & 0.44 & $+1$ / $+2\%$ & -- \\
\midrule
& TiRex-2 $+$ v4 recipe (one round) & $-29\%$ & $-42\%$ & 1.04 & 0.53 & $+6$ / $+6\%$ & -- \\
& TiRex-2 $+$ v4 recipe $+$ KS replay 3{,}000 & $-28\%$ & $-32\%$ & 1.03 & -- & $+2$ / $+4\%$ & -- \\
\bottomrule
\end{tabular}
\end{table*}
Causal Chambers \citep{gamella2025causal} are physical systems whose actuators are set by random walks, so the pooled data are experimental and the effect of an actuator on a sensor is identified by pooled regression; within a 576-step window, however, two random walks are often spuriously correlated, so windows span a range of in-window confounding exactly as our regimes do. In the \emph{wind tunnel} the intake load is the control $X$, the exhaust load the confounder $Z$ and the downwind relative pressure $Y$ (identified effect 16.4; 722 windows); in the \emph{light tunnel} $X$ is the red LED current, $Z$ the green, $Y$ a colour sensor (identified effect 1.87; 80 windows).

Three findings (\cref{fig:real}, \cref{tab:real}). The least-squares reference behaves as the theory says: slope 1.00 on the predicted bias under \Vo, 0.00 under \Vt\ (0.91 and $-0.02$ when the roles of the two loads are swapped, identified effect $-20.1$). The pretrained models exhibit a failure the synthetic grid cannot reveal: their response is 0.2--0.5$\times$ the identified effect on the wind tunnel (0.4--0.8$\times$ with the roles swapped) and 0.6--0.8$\times$ on the light tunnel, and it does not vary with the in-window confounding at any visibility (slopes $-0.03$ to $0.17$). This is not adjustment but indifference to the covariate---the preregistered pattern~D on real data. Post-training changes the response scale in both directions: on the wind tunnel it rises to 0.65--0.73$\times$ with the expected visibility structure (slope 0.41 under \Vz, where least squares has 0.43; 0.22 under \Vo; 0.03 under \Vt\ with RMSE 8.3 against 7.9 for least squares), while on the light tunnel it overshoots to 1.6--2.1$\times$. The repaired model responds to interventions; its scale is system dependent, which is the open problem we leave.

% =====================================================================
\section{Repair and Its Cost}
\label{sec:repair}

\begin{figure*}[t]
\centering
\includegraphics[width=\textwidth]{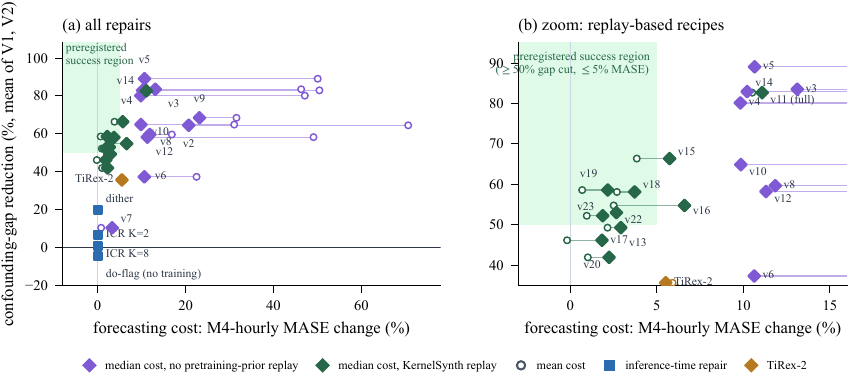}
\caption{\textbf{The repair trade-off.} Gap reduction against forecasting cost for every variant of \cref{tab:repair} (filled diamonds: median cost; open circles: mean; ICR = in-context randomization). (a)~Variants without pretraining-prior replay form a curve on which reductions above 50\% cost at least 10\% median MASE. (b)~KernelSynth replay (green) removes the mean-cost tail at full-parameter learning rates (v11) and, with low-rank adaptation, moves the curve into the preregistered region (v18, v19).}
\label{fig:tradeoff}
\end{figure*}

\Cref{prop:bayes} leaves two exits: supply the adjustment set and make the model use it, or change the training distribution. We take both (\cref{tab:repair}, \cref{fig:tradeoff}).

\paragraph{Inference-time repairs do not make in-context learning causal.}
Context dither \citep{won2026attenuated} reduces the \Vt\ gap by a third and the \Vo\ gap not at all. \emph{In-context randomization}---appending $K$ demonstration series from the randomized regime to the group, a zero-training analogue of few-shot prompting---raises $\bint$ from 0.75 to 0.90 but leaves $\bobs$ unchanged: it de-shrinks, it does not de-confound; the model does not infer from randomized demonstrations that it should adjust its own estimate. A ``do'' flag without training has no effect. None reaches the preregistered 50\%: in-context learning will not do causal identification on its own.

\paragraph{Interventional post-training re-maps the latent information.}
\emph{Interventional post-training} (\IPT) fine-tunes the forecaster for 2{,}000 steps on 3{,}000 series of length $L{+}H$ from the randomized regime of \eqref{eq:scm} with $\alpha,\gamma\sim U[0,1.5]$, $X$ (and optionally $Z$) as covariates and the ordinary forecasting loss (5--8 minutes on a laptop GPU). Two ablations fix the recipe. \emph{Training support determines the prior}: with $\beta\sim U[0.5,1.5]$ (v1) the model learns that $X$ always helps and reports a spurious $0.24$ at $\beta=0$ even under randomized context; including zero and negative effects (v2) removes it. \emph{Training visibility determines what is repaired}: \Vo-only training repairs \Vo\ ($-85\%$) and leaves \Vt\ half-done; \Vt-only training repairs \Vt\ ($-97\%$) but makes the model dependent on $Z$'s future; mixing both (v4) repairs both ($-78\%$, $-82\%$; $76\pm2$, $88\pm5$ over three seeds).
The repair is not memorisation of the training range. With confounding strengths outside the prior ($\alpha,\gamma\in\{2,2.5\}$) the reduction is unchanged ($-78\%$, $-79\%$; \cref{tab:ood}); at a $4\sigma$ intervention it is $-64\%/-54\%$; on the nonlinear structural model of \cref{sec:nonlin}, which the training data never contained, the gaps fall from 0.57/0.58 to 0.15/0.02; on DoTime effect-direction accuracy rises from 0.684 to 0.709 (back-door 0.68$\to$0.72) at unchanged NMSE; and the encoder itself changes (\cref{sec:probe}). The boundary is sign: at $\beta=-1$ with positive confounding the reduction is $-15\%/-44\%$, and training with signed $\gamma$ (v5) does not fix it while over-correcting elsewhere ($\bobs$ 1.26, spurious effect back to 0.12); when the direct and back-door paths have opposite signs and nearly cancel, the learned shortcut does not separate them as least squares does. A second boundary is the i.i.d.\ design of \cref{sec:ident}, where the full-parameter model recovers 48\% of the identifiable difference against 75\% before.

\paragraph{The cost is a replay-distribution effect, and it can be removed.}
The preregistered criterion requires a $\ge50\%$ gap reduction at $\le5\%$ forecasting degradation on M4-hourly \citep{makridakis2020m4}. The cost of v4 is $+47\%$ in mean and $+10\%$ in median MASE over 200 series (an evaluation bug initially hid it; amendment A3). \Cref{fig:tradeoff} places every variant on the resulting curve. Lower learning rates (v6), low-rank adaptation (v7, v8) and an L2-SP penalty toward the pretrained weights (v14) move along the curve, not off it; replaying \emph{real} univariate series (v9) makes forecasting worse ($+32\%/+23\%$), reproducing the finding of \citet{won2026attenuated} that M4 replay did not prevent forgetting, and rehearsing observational-regime series (v12) does not help. What moves the curve is the replay \emph{source}. Mixing in 3{,}000 series drawn from a KernelSynth GP prior---the family Chronos was pretrained on \citep{ansari2024chronos}---keeps the full-parameter gap reduction (v11: $-83\%/-82\%$) while cutting the mean cost from $+47\%$ to $+10\%$: the catastrophic tail of the v4 cost was forgetting of the pretraining distribution, not of forecasting. Combining that replay with low-rank adaptation brings the cost inside the budget and, with learning rate and replay size tuned, the gap reduction above the bar: v18 (lr $1.5\times10^{-4}$, 3{,}000 replay series) reduces the gap by $59\%/58\%$ at $+2.7\%/+3.7\%$ MASE, and v19 (lr $2\times10^{-4}$, 4{,}500 replay series) by $54\%/63\%$ at $+0.7\%/+2.2\%$. Both satisfy the criterion on the mean and on the median reading, the neighbouring settings (v13, v15--v17) bracket them on each bar, so the region is reached by a family of recipes, and over three seeds v19 reduces the gap by $52\pm1\%$ (\Vo) and $61\pm6\%$ (\Vt) at $+1.4\pm0.5\%$ mean and $+2.9\pm1.3\%$ median MASE, every seed inside the region. The cost of interventional post-training was never the interventional supervision; it was the absence of the pretraining distribution during training.

\paragraph{What the cheap recipe buys, and what it does not.}
The low-rank and the full-parameter recipes repair different things, and the distinction is the practical content of \cref{tab:repair}. v19 moves Chronos-2 from super-associational to regression-level: under \Vo\ its observational estimate falls from 1.72 to 1.46 (least squares: 1.42) and its randomized estimate rises from 0.75 to 1.01, so the front-loaded excess and the shrinkage are gone, while the omitted-variable bias that regression itself carries remains (spurious effect at $\beta=0$: 0.41 against 0.47 before and 0.25 for v4). The full-parameter recipes (v4, v11) go below the regression line ($\bobs$ 1.07--1.22) at a higher cost. Where v19 transfers, it transfers like v4: $-62\%/-61\%$ outside the confounding prior, $-54\%/-79\%$ at $L=2048$, $-45\%/-53\%/-73\%$ on the nonlinear model, DoTime effect-direction accuracy 0.710 (back-door 0.73) with the best NMSE of any model (0.433), and 61\% of the identifiable difference in the i.i.d.\ design of \cref{sec:ident} (v4: 48\%). The same holds for negative effects once they are read in estimates rather than percentages. At $\beta=-1$ the pretrained model's gap is small (0.38) only because its shrinkage ($\bint=-0.80$) offsets part of its confounding ($\bobs=-0.42$); v19 moves both to where it moves them at $\beta=+1$---$\bobs$ to the regression line ($-0.55$; least squares $-0.58$) and $\bint$ to the truth ($-0.95$)---so the apparent ``no reduction'' ($-6\%$) is an artefact of the gap-reduction metric at a cell where the baseline gap was itself attenuated. Training with signed confounding (v21) is unnecessary and costly ($-36\%/-26\%$ in distribution). The remaining boundaries are large interventions ($-34\%/-17\%$ at $4\sigma$) and the real systems, where v19 under-responds on both wind-tunnel configurations (0.42--0.56$\times$ and 0.69--0.85$\times$, slope $\approx0$) and over-responds on the light tunnel (1.15--1.52$\times$), by less than v4 in both directions; randomizing the noise scales and dynamics of the training family trades calibration against in-distribution repair: wide randomization (v20) brings the light-tunnel response to 1.16--1.19$\times$ but halves the \Vo\ gap reduction ($-29\%$), while randomizing only the $X$ and $Y$ scales (v22) keeps the reduction at the bar ($-50\%/-56\%$, $+2.7\%$ MASE) and moves the light-tunnel response to 1.06--1.41$\times$; neither changes the wind-tunnel under-response, which is a property of the pretrained model's response scale that no variant we trained repairs. Explicit counterfactual supervision through a ``do'' flag on observational contexts (v10) reduces the gap for the wrong reason: $\bobs$ barely moves (1.72$\to$1.57), the gap closes because $\bint$ overshoots to 1.13, and the spurious effect at $\beta=0$ is unchanged (0.42); regime training beats an interface. On TiRex-2 one round of the v4 recipe halves the covariate-dominance collapse ($\bint$ under \Vt: 0.50$\to$0.74) and reduces the gaps by 29--42\% at $+6\%$ MASE; adding KernelSynth replay cuts that cost to $+2\%/+4\%$ at a similar reduction, so the replay diagnosis transfers across architectures even where one round of full-parameter training does not reach the bar.

% =====================================================================
\section{Discussion}
\label{sec:discussion}
Covariate-aware TSFMs are, as \cref{prop:bayes} predicts and four models confirm, associational what-if engines: their immediate response is \emph{more} confounded than regression, on linear and nonlinear structural models alike; they carry the regime and the effect in their representations and do not use them; they fail to adjust in three characteristic ways when adjustment is possible, or do not respond at all; on real sensors they under-respond rather than adjust. A forecast difference from a pretrained covariate-aware TSFM is an estimate of the observational slope, front-loaded, and should be used as an effect only on experimental context.

\textbf{Limitations.}
The propositions are proved for linear-Gaussian models; nonlinearity is covered empirically (\cref{sec:nonlin}, DoTime, two physical systems). Post-training transfers outside its prior, to the nonlinear model, to DoTime and to three physical configurations, but one round on TiRex-2 is the only cross-architecture evidence. The low-rank recipe removes super-association and shrinkage but not the regression-level omitted-variable bias; calibrating the repaired response scale across physical systems trades against in-distribution repair in our hands, and forecasting cost is measured on one benchmark.

\textbf{Conclusion.}
We give the field a measurement and a recipe: the confounding gap with its closed-form reference and compass, on which any covariate-aware forecaster can be placed without ground truth, and interventional post-training with pretraining-prior replay, which moves a model toward the origin at a forecasting cost within the preregistered budget.

\section*{Impact Statement}
What-if forecasts from foundation models are increasingly used to plan prices, doses, loads and policies. This paper shows that four widely used models answer such queries associationally, and provides a test practitioners can run without ground truth. The main risk we see is the opposite one: that a post-trained model, now responsive to interventions but uncalibrated, is trusted more than its predecessor. The recipe removes most of the confounding gap on the synthetic family, not all of it, and its response scale on real systems is not yet calibrated. We report the forecasting cost and the light-tunnel overshoot for that reason, and recommend that any deployed what-if interface carry the confounding gap measured on the deployment data.

\bibliography{refs}

@book{pearl2009causality,
  title={Causality: Models, Reasoning, and Inference},
  author={Pearl, Judea},
  edition={2},
  year={2009},
  publisher={Cambridge University Press}
}

@incollection{bareinboim2022hierarchy,
  title={On {P}earl's hierarchy and the foundations of causal inference},
  author={Bareinboim, Elias and Correa, Juan D. and Ibeling, Duligur and Icard, Thomas},
  booktitle={Probabilistic and Causal Inference: The Works of Judea Pearl},
  pages={507--556},
  year={2022},
  publisher={ACM}
}

@book{angrist2009mostly,
  title={Mostly Harmless Econometrics: An Empiricist's Companion},
  author={Angrist, Joshua D. and Pischke, J{\"o}rn-Steffen},
  year={2009},
  publisher={Princeton University Press}
}

@article{lovell1963seasonal,
  title={Seasonal adjustment of economic time series and multiple regression analysis},
  author={Lovell, Michael C.},
  journal={Journal of the American Statistical Association},
  volume={58},
  number={304},
  pages={993--1010},
  year={1963}
}

@article{thumm2026dotime,
  title={{DoTime}: A synthetic benchmark generator for interventional and counterfactual time series},
  author={Thumm, Dennis and Anthony, Billy Tim and Chen, Ying},
  journal={arXiv preprint arXiv:2607.27263},
  year={2026}
}

@article{thumm2026priors,
  title={Interventional time series priors for causal foundation models},
  author={Thumm, Dennis and Chen, Ying},
  journal={arXiv preprint arXiv:2603.11090},
  year={2026},
  note={ICLR 2026 Workshop on Time Series in the Age of Large Models}
}

@article{gamella2025causal,
  title={Causal chambers as a real-world physical testbed for {AI} methodology},
  author={Gamella, Juan L. and Peters, Jonas and B{\"u}hlmann, Peter},
  journal={Nature Machine Intelligence},
  volume={7},
  pages={107--118},
  year={2025}
}

@article{ansari2025chronos2,
  title={{Chronos-2}: From univariate to universal forecasting},
  author={Ansari, Abdul Fatir and Shchur, Oleksandr and K{\"u}ken, Jaris and Auer, Andreas and Han, Boran and Mercado, Pedro and Rangapuram, Syama Sundar and Shen, Huibin and Stella, Lorenzo and Zhang, Xiyuan and Goswami, Mononito and Kapoor, Shubham and Maddix, Danielle C. and Guerron, Pablo and Hu, Tony and Yin, Junming and Erickson, Nick and Desai, Prateek Mutalik and Wang, Hao and Rangwala, Huzefa and Karypis, George and Wang, Yuyang and Bohlke-Schneider, Michael},
  journal={arXiv preprint arXiv:2510.15821},
  year={2025}
}

@article{ansari2024chronos,
  title={Chronos: Learning the language of time series},
  author={Ansari, Abdul Fatir and Stella, Lorenzo and Turkmen, Caner and Zhang, Xiyuan and Mercado, Pedro and Shen, Huibin and Shchur, Oleksandr and Rangapuram, Syama Sundar and Arango, Sebastian Pineda and Kapoor, Shubham and others},
  journal={Transactions on Machine Learning Research},
  year={2024}
}

@inproceedings{das2024decoder,
  title={A decoder-only foundation model for time-series forecasting},
  author={Das, Abhimanyu and Kong, Weihao and Sen, Rajat and Zhou, Yichen},
  booktitle={International Conference on Machine Learning},
  year={2024}
}

@article{hoo2025tables,
  title={From tables to time: Extending {TabPFN-v2} to time series forecasting},
  author={Hoo, Shi Bin and M{\"u}ller, Samuel and Salinas, David and Hutter, Frank},
  journal={arXiv preprint arXiv:2501.02945},
  year={2025}
}

@inproceedings{muller2022transformers,
  title={Transformers can do {B}ayesian inference},
  author={M{\"u}ller, Samuel and Hollmann, Noah and Pineda Arango, Sebastian and Grabocka, Josif and Hutter, Frank},
  booktitle={International Conference on Learning Representations},
  year={2022}
}

@article{tirex2_2026,
  title={{TiRex-2}: Generalizing {TiRex} to multivariate data and streaming},
  author={Podest, Patrick and Pichler, Marco and B{\"u}rger, Elias and Z{\'o}lyomi, Levente and Voggenberger, Bernhard and Berghammer, Wilhelm and Klotz, Daniel and B{\"o}ck, Sebastian and Klambauer, G{\"u}nter and Hochreiter, Sepp},
  journal={arXiv preprint arXiv:2607.01204},
  year={2026}
}

@inproceedings{beck2024xlstm,
  title={{xLSTM}: Extended long short-term memory},
  author={Beck, Maximilian and P{\"o}ppel, Korbinian and Spanring, Markus and Auer, Andreas and Prudnikova, Oleksandra and Kopp, Michael and Klambauer, G{\"u}nter and Brandstetter, Johannes and Hochreiter, Sepp},
  booktitle={Advances in Neural Information Processing Systems},
  year={2024}
}

@article{leakage2025rethinking,
  title={Rethinking evaluation in the era of time series foundation models: (Un)known information leakage challenges},
  author={Meyer, Marcel and Kaltenpoth, Sascha and Zalipski, Kevin and M{\"u}ller, Oliver},
  journal={arXiv preprint arXiv:2510.13654},
  year={2025}
}

@article{won2026attenuated,
  title={Attenuated in-context identification in time-series foundation models: Diagnosis under counterfactual inputs and repair by synthetic forced-system fine-tuning},
  author={Won, Hong-In},
  journal={arXiv preprint arXiv:2610.08118},
  year={2026}
}

@article{jander2026causal,
  title={Causal analysis for time series foundation models},
  author={Jander, Mathis and van Heeswijk, Wouter and Mes, Martijn},
  journal={arXiv preprint arXiv:2608.24303},
  year={2026}
}

@inproceedings{akyurek2023what,
  title={What learning algorithm is in-context learning? {I}nvestigations with linear models},
  author={Aky{\"u}rek, Ekin and Schuurmans, Dale and Andreas, Jacob and Ma, Tengyu and Zhou, Denny},
  booktitle={International Conference on Learning Representations},
  year={2023}
}

@inproceedings{vonoswald2023transformers,
  title={Transformers learn in-context by gradient descent},
  author={von Oswald, Johannes and Niklasson, Eyvind and Randazzo, Ettore and Sacramento, Jo{\~a}o and Mordvintsev, Alexander and Zhmoginov, Andrey and Vladymyrov, Max},
  booktitle={International Conference on Machine Learning},
  year={2023}
}

@inproceedings{garg2022what,
  title={What can transformers learn in-context? {A} case study of simple function classes},
  author={Garg, Shivam and Tsipras, Dimitris and Liang, Percy and Valiant, Gregory},
  booktitle={Advances in Neural Information Processing Systems},
  year={2022}
}

@article{hoerl1970ridge,
  title={Ridge regression: Biased estimation for nonorthogonal problems},
  author={Hoerl, Arthur E. and Kennard, Robert W.},
  journal={Technometrics},
  volume={12},
  number={1},
  pages={55--67},
  year={1970}
}

@article{illusion2026intervention,
  title={The illusion of intervention: Your {LLM}-simulated experiment is an observational study},
  author={Lin, Victoria and Yun, Taedong and Matari{\'c}, Maja and Canny, John and Gretton, Arthur and D'Amour, Alexander},
  journal={arXiv preprint arXiv:2605.20767},
  year={2026}
}

@article{causalarena2026,
  title={{CausalArena}: Benchmarking causal discovery in the foundation model era},
  author={Li, Zi-Rong and Liu, Si-Yang and Wang, Tian-Zuo and Ye, Han-Jia},
  journal={arXiv preprint arXiv:2609.11897},
  year={2026}
}

@article{kaddour2022causal,
  title={Causal machine learning: A survey and open problems},
  author={Kaddour, Jean and Lynch, Aengus and Liu, Qi and Kusner, Matt J. and Silva, Ricardo},
  journal={arXiv preprint arXiv:2206.15475},
  year={2022}
}

@article{nosek2018preregistration,
  title={The preregistration revolution},
  author={Nosek, Brian A. and Ebersole, Charles R. and DeHaven, Alexander C. and Mellor, David T.},
  journal={Proceedings of the National Academy of Sciences},
  volume={115},
  number={11},
  pages={2600--2606},
  year={2018}
}

@article{efron1979bootstrap,
  title={Bootstrap methods: Another look at the jackknife},
  author={Efron, Bradley},
  journal={The Annals of Statistics},
  volume={7},
  number={1},
  pages={1--26},
  year={1979}
}

@article{epidemic2026benchmark,
  title={Benchmarking counterfactual prediction in epidemic time series with time-varying interventions},
  author={Mu, Wenhao and Yan, Facundo and Mumssen, Anik and Eisenberg, Marisa and Rodr{\'i}guez, Alexander},
  journal={Proceedings of the 32nd ACM SIGKDD Conference on Knowledge Discovery and Data Mining},
  year={2026},
  note={arXiv:2606.05692}
}

@inproceedings{hu2022lora,
  title={{LoRA}: Low-rank adaptation of large language models},
  author={Hu, Edward J. and Shen, Yelong and Wallis, Phillip and Allen-Zhu, Zeyuan and Li, Yuanzhi and Wang, Shean and Wang, Lu and Chen, Weizhu},
  booktitle={International Conference on Learning Representations},
  year={2022}
}

@article{kirkpatrick2017overcoming,
  title={Overcoming catastrophic forgetting in neural networks},
  author={Kirkpatrick, James and Pascanu, Razvan and Rabinowitz, Neil and Veness, Joel and Desjardins, Guillaume and Rusu, Andrei A. and Milan, Kieran and Quan, John and Ramalho, Tiago and Grabska-Barwinska, Agnieszka and others},
  journal={Proceedings of the National Academy of Sciences},
  volume={114},
  number={13},
  pages={3521--3526},
  year={2017}
}

@inproceedings{rolnick2019experience,
  title={Experience replay for continual learning},
  author={Rolnick, David and Ahuja, Arun and Schwarz, Jonathan and Lillicrap, Timothy and Wayne, Gregory},
  booktitle={Advances in Neural Information Processing Systems},
  year={2019}
}

@article{makridakis2020m4,
  title={The {M4} competition: 100,000 time series and 61 forecasting methods},
  author={Makridakis, Spyros and Spiliotis, Evangelos and Assimakopoulos, Vassilios},
  journal={International Journal of Forecasting},
  volume={36},
  number={1},
  pages={54--74},
  year={2020}
}

@article{hyndman2006another,
  title={Another look at measures of forecast accuracy},
  author={Hyndman, Rob J. and Koehler, Anne B.},
  journal={International Journal of Forecasting},
  volume={22},
  number={4},
  pages={679--688},
  year={2006}
}

@inproceedings{alain2017understanding,
  title={Understanding intermediate layers using linear classifier probes},
  author={Alain, Guillaume and Bengio, Yoshua},
  booktitle={International Conference on Learning Representations, Workshop Track},
  year={2017}
}

@article{scholkopf2021toward,
  title={Toward causal representation learning},
  author={Sch{\"o}lkopf, Bernhard and Locatello, Francesco and Bauer, Stefan and Ke, Nan Rosemary and Kalchbrenner, Nal and Goyal, Anirudh and Bengio, Yoshua},
  journal={Proceedings of the IEEE},
  volume={109},
  number={5},
  pages={612--634},
  year={2021}
}

@article{wasserstein2016asa,
  title={The {ASA} statement on p-values: Context, process, and purpose},
  author={Wasserstein, Ronald L. and Lazar, Nicole A.},
  journal={The American Statistician},
  volume={70},
  number={2},
  pages={129--133},
  year={2016}
}
\bibliographystyle{icml2026}

% =====================================================================
\newpage
\appendix
\onecolumn
\raggedbottom

\section{Proofs}
\label{app:proofs}

\subsection{Proof of \cref{prop:bayes}}
Bayes optimality under squared loss for the observational distribution $P$ gives $f(c,x)=\E_P[Y_{\mathrm{fut}}\mid c,X_{\mathrm{fut}}=x]$. Write the contrast as
\[
f(c,x+\delta)-f(c,x)=\big[\E_P[Y\mid c,\dob(x+\delta)]-\E_P[Y\mid c,\dob(x)]\big]+\mathcal{B}(c,x,\delta),
\]
\[
\mathcal{B}(c,x,\delta):=\big(\E_P[Y\mid c,X=x+\delta]-\E_P[Y\mid c,\dob(x+\delta)]\big)-\big(\E_P[Y\mid c,X=x]-\E_P[Y\mid c,\dob(x)]\big),
\]
an identity that defines $\mathcal{B}$. Let $C\subseteq c$ be the variables $f$ conditions on. If $C$ satisfies the back-door criterion relative to $(X_{\mathrm{fut}},Y_{\mathrm{fut}})$---no element of $C$ is a descendant of $X_{\mathrm{fut}}$, and $C$ blocks every path from $X_{\mathrm{fut}}$ to $Y_{\mathrm{fut}}$ that starts with an arrow into $X_{\mathrm{fut}}$---then Pearl's back-door theorem \citep[Thm.~3.3.2]{pearl2009causality} gives $P(Y\mid C,\dob(u))=P(Y\mid C,X=u)$ for every $u$, hence $\mathcal{B}=0$. Conditioning on the past of the process never creates descendants of $X_{\mathrm{fut}}$, so the first condition holds for any $C\subseteq c$. For the model \eqref{eq:scm} under \Vz, $C$ contains no $Z$ and the path $X_{\mathrm{fut}}\leftarrow Z\rightarrow Y_{\mathrm{fut}}$ is open whenever $\alpha\gamma\neq0$; under \Vo, the future of $Z$ along the horizon is not in $C$ and the same path through $Z_{L+h-1}$, $h\ge1$, is open; under \Vt\ the path is blocked by the known-future $Z$. The linear-Gaussian expressions for $\mathcal{B}$ in the open cases are the terms derived in \cref{prop:cg}. \qed

\subsection{Proof of \cref{prop:cg}}
We work with population (large-$L$) regressions; stationarity of \eqref{eq:scm} holds for $|\rho_\cdot|<1$, and sample moments converge by ergodicity.

\paragraph{\Vz.}
The predictor regresses $Y_t$ on $W_t=(1,Y_{t-1},X_t)$. Substituting \eqref{eq:scm}, $Y_t=W_t\theta+\gamma Z_{t-1}+\varepsilon^y_t$ with $\theta=(0,\rho_y,\beta)$. The noise $\varepsilon^y_t$ is independent of $W_t$ (it enters neither $Y_{t-1}$ nor $X_t$), so the population coefficient vector is $\theta+\gamma\,\Pi$, where $\Pi=\E[W_tW_t^\top]^{-1}\E[W_tZ_{t-1}]$ is the coefficient vector of the regression of the omitted $Z_{t-1}$ on $W_t$ (omitted-variable formula; \citealp{angrist2009mostly}). Its $X_t$ component is $\pi_X$, which by the Frisch--Waugh--Lovell theorem \citep{lovell1963seasonal} equals the simple regression coefficient of $Z_{t-1}$ on the residual of $X_t$ after projection on $(1,Y_{t-1})$. Hence $\bone=\beta+\gamma\pi_X(\text{regime})$. In the observational regime $\pi_X>0$ because $X_t$ loads on $Z_{t-1}$ with coefficient $\alpha$. In the randomized regime $\Cov(X_t,Z_{t-1})=0$, but $\Cov(Y_{t-1},Z_{t-1})\neq0$ (through $\gamma Z_{t-2}$ and $\rho_z$) and $\Cov(X_t,Y_{t-1})\neq0$ (through $\rho_x$ and $\beta$), so the partial coefficient $\pi_X(\mathrm{int})$ is nonzero and negative: controlling for $Y_{t-1}$ induces a correlation between $X_t$ and $Z_{t-1}$. The confounding gap is therefore $\gamma[\pi_X(\mathrm{obs})-\pi_X(\mathrm{int})]$; omitting the randomized-regime term overstates the predicted gap (slope 1.15 instead of 1.00 in our numerical check).

\paragraph{\Vo.}
The predictor regresses $Y_t$ on $(1,Y_{t-1},X_t,Z_{t-1})$. Since $\varepsilon^y_t$ is independent of all four regressors, the population coefficients are exactly $(0,\rho_y,\beta,\gamma)$ in both regimes. To forecast the horizon it needs the future of $Z$, which is not supplied; the reference predictor imputes it by the in-context linear projection of $Z_{t-1}$ on the contemporaneous $X_t$, $\hat Z=g\,X$ with $g=\Cov(Z_{t-1},X_t)/\Var(X_t)$, so that the response of the forecast to a unit shift of the supplied $X$ path is $\beta+\gamma g$. In the randomized regime $X_t$ is independent of $Z_{t-1}$, $g(\mathrm{int})=0$, and the response is $\beta$; in the observational regime $g(\mathrm{obs})=\alpha\Var(Z)/\Var(X)\neq0$. Hence $\CG=\gamma g(\mathrm{obs})$. This is the abduction step of counterfactual reasoning \citep[Thm.~7.1.7]{pearl2009causality} performed with the observational, rather than the interventional, model of $Z$: the intervention on $X$ is treated as evidence about $Z$. The imputation rule is a modelling choice for the reference predictor; empirically the TSFMs' \Vo\ gaps coincide with their \Vz\ gaps (\cref{tab:e1}), so the models do not use past-only covariates for adjustment at all.

\paragraph{\Vt.}
The same regression is used, and the horizon values of $Z$ are supplied and identical in both arms of the query (they are unaffected by $\dob(X)$). The forecast difference is then $\beta\delta$ at the first intervened step in both regimes, so $\bone=\beta$ and $\CG=0$. \qed

\paragraph{Numerical verification.}
An explicit least-squares forecaster implementing the three rules reproduces the proposition on the full grid with slope $0.999/1.000$ and $R^2=1.000$ under \Vz/\Vo, and $\max|\CG|=0.008$ under \Vt\ (\cref{fig:refline}, right panel; \texttt{sim/theory\_check.py}). The verification also caught a harness error before any foundation model was run: the originally preregistered cumulative coefficient mixes the contemporaneous coefficient with biased autoregressive dynamics under \Vz\ and does not satisfy the proposition (slope 0.05 for the control); the impact coefficient does (amendment A1).

\subsection{Proof of \cref{prop:shrink}}
After partialling out the other regressors (Frisch--Waugh--Lovell), the coefficient on the control is a univariate regression of residualised $Y$ on residualised $X$, $\tilde x$, with $\sum_t\tilde x_t^2=Lv$. Ridge regression with penalty $\lambda$ on that coordinate gives $b_\lambda=\sum\tilde x_t\tilde y_t/(Lv+\lambda)=b_{\mathrm{OLS}}\cdot v/(v+\lambda/L)$ \citep{hoerl1970ridge}. The factor is in $(0,1)$, increasing in $L$ and in $v$, and tends to 1. Under randomized context $b_{\mathrm{OLS}}\to\beta$, so $S=b_\lambda/\beta<1$ with the stated monotonicity; under observational context $b_{\mathrm{OLS}}\to\beta+\text{bias}$ with the bias of \cref{prop:cg}, which does not depend on $L$. \qed

\subsection{Proof of \cref{prop:static} and a general corollary}
\label{app:static}
Regress $Y_t=\rho_yY_{t-1}+\beta X_t+\gamma Z_{t-1}+\varepsilon^y_t$ on $(1,X_t)$ alone. Both $Y_{t-1}$ and $Z_{t-1}$ are omitted, and $\varepsilon^y_t$ is independent of $X_t$, so the population slope is $\beta$ plus each omitted term's coefficient times its regression on $X_t$: $b_{\mathrm{s}}=\beta+\gamma\Cov(Z_{t-1},X_t)/\Var(X_t)+\rho_y\Cov(Y_{t-1},X_t)/\Var(X_t)$. Under the randomized regime the first ratio is zero but the second is not ($Y_{t-1}$ and $X_t$ are correlated through $\rho_x$ and $\beta$), so $b_{\mathrm{s}}(\mathrm{int})\neq\beta$; the gap subtracts the two regimes. Monte Carlo on $2\times10^5$-step series reproduces the formula to three decimals at every grid cell. The ratio $\CG_{\mathrm{s}}/\CG$ is computed on the regenerated grid contexts (\texttt{paper/static\_reference.py}); it is $2.1$ under \Vz\ and $2.6$ under \Vo\ for every confounded cell. \qed

\begin{corollary}[Zero gap under adjustment, general structural models]
\label{cor:general}
Let the data be generated by any structural model, and let $C$ be the set of context and horizon variables a predictor conditions on. If $C$ contains no descendant of $X_{\mathrm{fut}}$ and blocks every back-door path from $X_{\mathrm{fut}}$ to $Y_{\mathrm{fut}}$ in both regimes, then the Bayes-optimal predictor on $C$ has $\CG=0$.
\end{corollary}
\begin{proof}
By \cref{prop:bayes}, $\mathcal{B}=0$ in each regime, so the what-if contrast equals the interventional contrast $\E[Y\mid C,\dob(x+\delta)]-\E[Y\mid C,\dob(x)]$ in each regime. The interventional distribution of $Y_{\mathrm{fut}}$ given $C$ and $\dob(X_{\mathrm{fut}})$ does not depend on the mechanism generating $X$ (the $\dob$ operator removes it), hence is the same in the observational and the randomized regime, and the two contrasts coincide.
\end{proof}
The corollary is what the \Vt\ column of \cref{fig:refline} and the \Vt\ panel of \cref{fig:nonlin} test: a zero gap is the prediction for \emph{any} data-generating process once the confounder's future is supplied, linear or not.

\begin{table}[H]
\caption{\textbf{Dynamic versus static reference.} Slope, $R^2$ and RMSE of each model's measured impact gap on the \cref{prop:cg} (dynamic) and \cref{prop:static} (static) references over the confounded cells; last column: ratio of the two references.}
\label{tab:static}
\vskip 0.1in
\centering\small
\begin{tabular}{llcccccc}
\toprule
Model & Vis. & slope (dyn.) & $R^2$ & RMSE & slope (static) & $R^2$ & RMSE \\
\midrule
Chronos-2 & \Vz & 1.64 & 0.99 & 0.40 & 1.17 & 0.95 & 0.12 \\
Chronos-2 & \Vo & 1.78 & 0.83 & 0.46 & 1.07 & 0.94 & 0.10 \\
TiRex-2 & \Vz & 1.34 & 1.00 & 0.23 & 0.95 & 0.94 & 0.15 \\
TiRex-2 & \Vo & 1.56 & 0.83 & 0.35 & 0.94 & 0.94 & 0.15 \\
TabPFN-TS & \Vz & 1.00 & 0.98 & 0.04 & 0.71 & 0.84 & 0.40 \\
TimesFM-2.5 + XReg & \Vz & 1.26 & 0.90 & 0.40 & 1.01 & 1.00 & 0.01 \\
TimesFM-2.5 + XReg & \Vo & 1.62 & 0.96 & 0.54 & 1.01 & 1.00 & 0.01 \\
In-context OLS & \Vz\ / \Vo & 1.00 & 1.00 & 0.00 & 0.71 / 0.55 & 0.93 & 0.36 / 0.45 \\
\bottomrule
\end{tabular}
\end{table}

\paragraph{Scope of \cref{prop:shrink}.}
The proposition explains the monotone rise of $S$ with $L$ under randomized context (\cref{fig:shrink}) and the invariance of the gap. The dependence on dose $\delta$ that we observe ($\bint$ 0.72$\to$0.94 from $\delta=0.5\sigma$ to $4\sigma$) is not a property of linear ridge regression and indicates a nonlinear or quantised response in the models.

% ---------------------------------------------------------------------
\section{Experimental Details}
\label{app:details}

\paragraph{Generator.}
Equation \eqref{eq:scm} with a burn-in of 200 steps; context length $L$ and horizon $H$ are simulated continuously, the horizon twice with shared noise (factual and $\dob$ arm), which gives the exact level-three counterfactual effect path \citep{bareinboim2022hierarchy}. The intervention adds $\delta$ to the factual future path of $X$ from the onset onward (onset 0 unless stated); $Z$'s future is verified unchanged. Default parameters: $\beta=1$, $\rho_z=0.6$, $\rho_x=\rho_y=0.5$, unit noise.

\paragraph{Models and interfaces.}
Chronos-2 (\texttt{amazon/chronos-2}, \texttt{predict\_df} with \texttt{context\_df}/\texttt{future\_df}; past-only covariates appear in the context only, known-future covariates in both). TiRex-2 (\texttt{tirex-2} 0.3.1; tensors target $[1,T]$, past covariates $[V_p,T]$, future covariates $[V_f,T{+}H]$; median = quantile index 4). TabPFN-TS (\texttt{tabpfn-time-series} 1.3.0, local mode; covariates present in the context but absent from the future frame are dropped, so \Vo$\equiv$\Vz). In-context OLS: the three rules of \cref{prop:cg} on the context, applied along the horizon with an AR(1) carry. All models forecast both arms from the same context; the impact coefficient uses the first horizon step.

\paragraph{E2--E3.}
E2: $\beta=0$, $\alpha,\gamma\in\{0.5,1,1.5\}^2$, 600 episodes per cell; the false-positive rate is the fraction of cells whose 95\% bootstrap CI of the mean effect excludes 0; because $n$ is large we also report the magnitude. E3: $L\in\{128,512,2048\}$ and $\delta\in\{0.5,1,4\}\sigma$ on the same $3\times3$ cells; anticipation uses onset $H/2=32$ and the leak ratio $\lVert\Delta\hat y_{<\text{onset}}\rVert^2/\lVert\Delta\hat y_{\ge\text{onset}}\rVert^2$.

\paragraph{E4 DoTime.}
\texttt{dotime} 0.2.0, suite \texttt{dot-Identifiability-v1} 1.2.0, \texttt{evaluate(dir\_target="effect")}, diverged episodes dropped as in the official protocol; forecasters are wrapped as observational-context baselines that receive the intervened path as a known-future covariate.

\paragraph{E6 Causal Chambers.}
\texttt{wt\_walks\_v1} (experiments \texttt{actuators\_random\_walk\_*} and \texttt{loads\_hatch\_mix\_*}): $X=$\,\texttt{load\_in}, $Z=$\,\texttt{load\_out}, $Y=$\,\texttt{pressure\_downwind}$-$\texttt{pressure\_ambient} (relative pressure removes barometric drift; pooled $R^2$ 0.61 against 0.005 for absolute pressure). The second wind-tunnel configuration swaps the roles ($X=$\,\texttt{load\_out}, $Z=$\,\texttt{load\_in}; identified effect $-20.1$). \texttt{lt\_walks\_v1} (\texttt{actuators\_white}, \texttt{color\_mix}): $X=$\,\texttt{red}, $Z=$\,\texttt{green}, $Y=$\,\texttt{vis\_2}, additional adjustment for \texttt{blue} and the six polariser/LED settings (pooled $R^2$ 0.976). Windows of $L{+}H=576$ steps with stride 288; $\delta=0.2$ load units and 30 current units respectively (within the actuator range). The identified effect is the pooled static regression on the randomized-actuator data; the in-window reference is the window's own static least squares with and without $Z$, whose difference is the predicted confounding bias. Effects are horizon averages because the physical response is fast relative to the sampling interval. The epidemic counterfactual benchmark \citep{epidemic2026benchmark} listed in the preregistration as a second external source was not obtained in time and remains future work.

\paragraph{Deployment procedure (ETTh1).}
The audit needs no ground truth, only a regime proxy, so it can be run on deployment data. On ETTh1 (target \texttt{OT}, control \texttt{HUFL}, candidate confounder \texttt{HULL}, 59 windows of 576 steps, $\delta=1$\,sd of the control) the in-window predicted confounding bias has standard deviation 0.14 around $-0.01$, against 34 on the wind tunnel: the dataset offers almost no variation in confounding, so the gap is near zero for every model by construction ($|\CG_{\mathrm{real}}|\le0.18$) and the slopes on the predicted bias are dominated by noise. The informative reading is the disagreement in sign between the least-squares effect ($-0.25$) and the pretrained Chronos-2 effect ($+0.08$) at the same visibility. The procedure therefore returns a clear verdict---the what-if answers on this dataset cannot be validated as effects---and a diagnosis of why: without a confounded/unconfounded contrast in the deployment data, the only recourse is an experiment or the synthetic audit of this paper. Script: \texttt{run\_deploy\_demo.py}.

\paragraph{TimesFM-2.5.}
\texttt{timesfm} 3.0.2, \texttt{google/timesfm-2.5-200m-pytorch}, context 1024, continuous quantile head, no positivity inference. Univariate: covariates ignored. XReg: an in-context least-squares layer $Y_t\sim 1+X_t$ (\Vz) or $1+X_t+Z_{t-1}$ (\Vo, \Vt) fitted on the context, TimesFM forecasting the residual, the linear part added back on the horizon; under \Vo\ the future of $Z_{t-1}$ is imputed from $X$ as for the least-squares control.

\paragraph{E11 nonlinear model and KRR reference.}
$X_t=\rho_xX_{t-1}+\alpha\tanh(Z_{t-1})+\varepsilon^x_t$ and $Y_t=\rho_yY_{t-1}+1.5\beta\tanh(X_t)+\gamma(Z_{t-1}+0.3Z_{t-1}^2)+\varepsilon^y_t$, other settings as in \eqref{eq:scm}; $3\times3$ confounded cells, one seed, 200 episodes. The reference fits an RBF kernel ridge regression (median-heuristic bandwidth, ridge 1) of $Y_t$ on standardized $(Y_{t-1},X_t[,Z_{t-1}])$ and forecasts recursively, imputing $Z_{t-1}$ under \Vo\ by a second kernel regression on $X_t$; slopes in \cref{fig:nonlin} are of each model's per-cell gap on the reference's per-cell gap.

\paragraph{E7 matched pairs.}
I.i.d.\ linear-Gaussian models ($\rho=0$). Given $(a,\gamma,\beta,s_x,s_y)=(1,1,1,1,2.5)$ and a target $\beta_B$, we solve for $(\gamma_B,s_{y,B})$ so that $\Var X$, $\Cov(X,Y)$ and $\Var Y$ coincide (verified to the second decimal by Monte Carlo); 200 episodes per member; separation ratio $=(\bone^A-\bone^B)/(\beta_A-\beta_B)$ averaged over $\beta_B\in\{0.5,0.25,0\}$.

\paragraph{E9 probes.}
600 contexts (half per regime; $\alpha,\gamma\in\{0.5,1,1.5\}$; $\beta=0$ with probability $0.25$, else $U[-1.5,1.5]$) embedded with \texttt{Chronos2Pipeline.embed}; features are the target variate's [REG] token, the mean over patches, and their concatenation (reported). For TiRex-2 a forward hook on the post-stack normalisation captures the target variate's hidden states under \Vo; features are the last token, the token mean and their concatenation (last token reported). Logistic regression (regime) and ridge regression (effect, bias) with inner cross-validated regularisation; 5-fold held-out accuracy and $R^2$.

\begin{table}[H]
\caption{\textbf{Linear probes} (5-fold held-out) on pretrained and post-trained encoders of both architectures.}
\label{tab:probe}
\vskip 0.1in
\centering\small
\begin{tabular}{lcccc}
\toprule
Encoder & regime acc. & $\beta$ $R^2$ (rand.) & $\beta$ $R^2$ (obs.) & bias $R^2$ \\
\midrule
Chronos-2 pretrained & 0.96 & 0.80 & 0.71 & $-0.03$ \\
Chronos-2 + \IPT\ (v4) & 0.99 & 0.90 & 0.82 & 0.24 \\
TiRex-2 pretrained & 0.94 & 0.90 & 0.86 & 0.15 \\
TiRex-2 + \IPT\ (one round) & 0.96 & 0.94 & 0.87 & 0.24 \\
raw in-window OLS statistics & 1.00 & 1.00 & -- & -- \\
\bottomrule
\end{tabular}
\end{table}

\paragraph{Interventional post-training.}
3{,}000 series of length $L{+}H=576$ from the randomized regime with $\alpha\sim U[0,1.5]$, $\gamma\sim U[\gamma_{\min},1.5]$, $\beta=0$ with probability $0.25$ else $U[-1.5,1.5]$ (v1: $U[0.5,1.5]$); $X$ as known-future covariate; $Z$ past-only (v2), known-future (v3), or half/half (v4--v14); Chronos-2 \texttt{fit} with the ordinary pretraining loss, 2{,}000 steps, batch 32, AdamW, lr $10^{-5}$ unless stated, full-parameter or LoRA (library default configuration); seed 2026 (plus two further seeds for v4). v9 mixes 2{,}922 real univariate \texttt{m4\_daily} series as replay; v10 trains on observational contexts with a binary known-future \texttt{do} flag and both arms of the shared-noise counterfactual as targets ($\delta\sim U[-2,2]\sigma$); v11 and v13 mix 3{,}000 KernelSynth series \citep{ansari2024chronos} of length 576 (random sums and products of up to five kernels from a bank of periodic, linear, RBF, rational-quadratic, white and constant kernels, one draw from the GP prior each, standardized); v12 replaces half of the interventional series by observational-regime series without intervention; v14 adds $\lambda\sum_j(\theta_j-\theta_j^{\mathrm{pre}})^2$ with $\lambda=10^{-2}$ to the loss; v15--v19 vary the LoRA learning rate ($10^{-4}$ to $2\times10^{-4}$), steps and replay size (1{,}500--4{,}500); v20 draws the training SCMs' noise scales from $U[0.3,3]$ and autoregressive coefficients from $U[0,0.9]$ (v22: $s_x,s_y\sim U[0.5,2]$ only; v23: half of the series so randomized); v21 draws $\gamma\sim U[-1.5,1.5]$. TiRex-2: custom trainer with a pinball loss over the nine quantiles, 1{,}000 steps, batch 16, lr $10^{-5}$, optionally alternating interventional and KernelSynth replay batches.

\paragraph{Forecasting cost.}
MASE \citep{hyndman2006another} on 200 M4-hourly series \citep{makridakis2020m4} with the model's native context, horizon 48, no covariates; reported as mean and median over series of the ratio to the pretrained model. A synthetic univariate observational test set changes by less than 1\% for every variant.

% ---------------------------------------------------------------------
\section{Preregistration and Amendments}
\label{app:prereg}
The preregistration file, the amendment log and all code referenced by file name in this appendix will be released with the final version. The preregistration (\texttt{preregister.yaml}) froze, before any foundation model was run: the grid, the effect coefficient, the four E1 patterns (A associational: slope in $[0.6,1.4]$, $R^2\ge0.7$, $|\CG_{\Vt}|\le0.1\times$ drift; B fully associational: same slope, $|\CG_{\Vt}|\ge0.5\times$ drift; C partially causal: \Vo\ slope $\le0.5$ and $|\bint-\beta|\le0.2$; D unresponsive: $S\le0.1$), the E2 false-positive rule, the E3 leak threshold (0.05), the E4 protocol, and the E5 success rule ($\ge50\%$ gap reduction and $\le5\%$ MASE degradation). Three amendments:
\textbf{A1} (before any grid result): the cumulative coefficient was replaced by the impact coefficient as the primary readout because the least-squares control did not satisfy \cref{prop:cg} under \Vz\ with the former (slope 0.05) and did with the latter (1.000); thresholds unchanged.
\textbf{A2} (after E1): slopes above 1.4 were not anticipated; the frozen verdict (unclassified) is reported, with an exploratory label A$^+$ for slope $>1.4$ and $R^2>0.7$.
\textbf{A3} (after E5): the forecasting-cost script read the wrong axis of Chronos-2's $(1,H,Q)$ quantile output; the pretrained MASE was 6.8 instead of 0.72 and all costs were underestimated; corrected values are used throughout. Variants v1--v10 fail the second success criterion under the corrected values; v18 and v19 (\cref{sec:repair}), designed after the correction, meet both.

\section{Failure-Mode Classifier}
\label{app:classifier}
Inputs per model, from confounded cells: $s_{\Vo}$ (slope of measured on predicted gap under \Vo; \Vz\ for TabPFN-TS), $\rho=\CG_{\Vt}/\CG_{\Vo}$ on matched cells, $d_{\mathrm{int}}=\bint^{\Vt}-\bint^{\Vo}$, $e_{\mathrm{int}}=|\bint^{\Vt}-\beta|$, and the two gaps relative to the predicted drift. Rules, in order: \emph{unresponsive} (preregistered pattern D) if mean sensitivity $S<0.1$; \emph{causal} if $s_{\Vo}<0.5$, $\rho<0.2$, $e_{\mathrm{int}}<0.2$; \emph{near-causal} if both relative gaps $<0.35$ and $e_{\mathrm{int}}<0.2$; \emph{covariate dominance} if $d_{\mathrm{int}}<-0.2$; \emph{no adjustment} if $\rho\ge0.8$; \emph{partial adjustment} if $0.2\le\rho<0.8$; otherwise \emph{near-causal (adjusts)}. The qualifier A$^+$ is attached when $s_{\Vo}>1.4$. \Cref{tab:modes} gives the inputs and labels for the pretrained models and the main repaired model; dither, in-context randomization and v6--v8 remain partial adjustment, v4 seeds 1--2 and v5 are near-causal, TiRex-2 after post-training remains covariate dominance.

\begin{table}[H]
\caption{\textbf{Failure-mode classifier} applied to the E1 grid.}
\label{tab:modes}
\vskip 0.1in
\centering\small
\begin{tabular}{lcccc l}
\toprule
Model & $s_{\Vo}$ & $\rho$ & $d_{\mathrm{int}}$ & $e_{\mathrm{int}}$ & Mode \\
\midrule
In-context OLS & 1.00 & 0.00 & 0.00 & 0.00 & adjusts (\Vo\ associational by construction) \\
Chronos-2 & 1.78 & 0.58 & $+0.16$ & 0.03 & partial adjustment, A$^+$ \\
TabPFN-TS & 1.00 & 0.89 & $+0.01$ & 0.11 & no adjustment \\
TiRex-2 & 1.56 & 1.07 & $\mathbf{-0.34}$ & 0.37 & covariate dominance, A$^+$ \\
TimesFM-2.5 + XReg & 1.62 & 0.88 & $0.00$ & 0.32 & no adjustment, A$^+$ \\
TimesFM-2.5 (univariate) & 0.00 & -- & 0.00 & 1.00 & unresponsive (pattern D) \\
Chronos-2 + \IPT\ (v4) & 0.79 & 0.39 & $+0.05$ & 0.09 & partial adjustment \\
Chronos-2 + \IPT\ (v11) & 0.60 & 0.40 & $-0.05$ & 0.00 & partial adjustment \\
\bottomrule
\end{tabular}
\end{table}

% ---------------------------------------------------------------------
\section{Additional Results}
\label{app:extra}

\begin{table}[H]
\caption{\textbf{Cell-bootstrap 95\% confidence intervals for the impact slopes of \cref{tab:e1}} (2{,}000 resamples of the confounded cells).}
\label{tab:ci}
\vskip 0.1in
\centering\small
\begin{tabular}{llccc}
\toprule
Model & Vis. & cells & slope & 95\% CI \\
\midrule
Chronos-2 & \Vz & 16 & 1.64 & [1.57, 1.71] \\
Chronos-2 & \Vo & 16 & 1.78 & [1.36, 2.29] \\
TiRex-2 & \Vz & 16 & 1.34 & [1.31, 1.36] \\
TiRex-2 & \Vo & 16 & 1.56 & [1.21, 2.03] \\
TabPFN-TS & \Vz & 9 & 1.00 & [0.85, 1.13] \\
In-context OLS & \Vz\ / \Vo & 16 & 1.00 & [1.00, 1.00] \\
\bottomrule
\end{tabular}
\end{table}

\begin{figure}[H]
\centering
\includegraphics[width=0.5\textwidth]{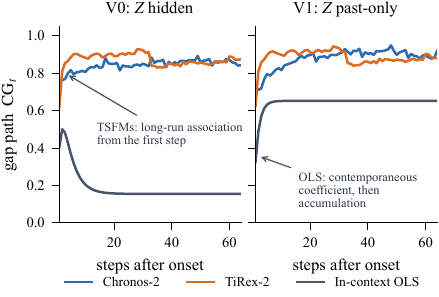}
\caption{\textbf{Gap path over the horizon} $\CG_t=\E_{\mathrm{obs}}[\Delta\hat y_t]/\delta-\E_{\mathrm{int}}[\Delta\hat y_t]/\delta$, averaged over the 16 confounded cells. In-context least squares answers with the contemporaneous coefficient at $t=1$ (0.40 under \Vz, 0.33 under \Vo) and then accumulates through the autoregressive carry (to 0.15 and 0.65). Chronos-2 and TiRex-2 answer with the long-run association from the first step: their gap is 0.6--0.8 at $t=1$ and 0.85--0.9 thereafter. This is why the impact slope exceeds 1 while the cumulative slope is close to 1.}
\label{fig:step}
\end{figure}

\begin{figure}[H]
\centering
\includegraphics[width=0.5\textwidth]{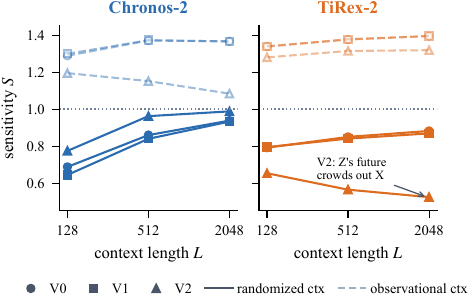}
\caption{\textbf{Sensitivity versus context length.} Under randomized context $S$ rises monotonically toward 1 with $L$ (\cref{prop:shrink}); under observational context it stays at 1.3--1.4 at every $L$: the confounding does not wash out. TiRex-2 under \Vt\ moves the other way as the known-future covariate increasingly dominates.}
\label{fig:shrink}
\end{figure}

\begin{table}[H]
\caption{\textbf{E3 sensitivity versus context length} (randomized context, $3\times3$ cells, seed 0).}
\vskip 0.1in
\centering\small
\begin{tabular}{llccc}
\toprule
Model & Vis. & $L=128$ & $L=512$ & $L=2048$ \\
\midrule
Chronos-2 & \Vz / \Vo / \Vt & 0.69 / 0.64 / 0.77 & 0.86 / 0.84 / 0.96 & 0.94 / 0.93 / 0.99 \\
TiRex-2 & \Vz / \Vo / \Vt & 0.79 / 0.80 / 0.65 & 0.85 / 0.84 / 0.56 & 0.88 / 0.87 / 0.52 \\
\bottomrule
\end{tabular}
\end{table}

\begin{table}[H]
\caption{\textbf{E8 transfer of \IPT\ outside the training prior} (\Vo\ / \Vt): full-parameter v4 (top) and the low-rank recipe v19 (bottom).}
\label{tab:ood}
\vskip 0.1in
\centering\small
\resizebox{\textwidth}{!}{%
\begin{tabular}{llcccc}
\toprule
Cell & Relation to training prior & pretrained $\CG$ & \IPT\ $\CG$ & reduction & effect RMSE (obs.) \\
\midrule
in-distribution $3\times3$ & -- & 0.98 / 0.55 & 0.22 / 0.10 & 78\% / 82\% & 0.76 $\to$ 0.15 \\
$\alpha,\gamma\in\{2,2.5\}$ & outside ($\le1.5$ in training) & 1.62 / 1.10 & 0.35 / 0.24 & \textbf{78\% / 79\%} & 1.11 $\to$ 0.23 \\
$\beta=0.5$ & inside & 0.83 / 0.37 & 0.21 / 0.15 & 74\% / 60\% & 0.63 $\to$ 0.20 \\
$\delta=4\sigma$ & intervention magnitude outside & 0.97 / 0.65 & 0.35 / 0.30 & 64\% / 54\% & 0.89 $\to$ 0.13 \\
$\beta=-1$ & inside, but confounding sign fixed positive & 0.38 / 0.39 & 0.33 / 0.22 & \textbf{15\% / 44\%} & 0.63 $\to$ 0.46 \\
$\beta=-1$, v5 (signed $\gamma$) & -- & 0.38 / 0.39 & 0.54 / 0.33 & $-40\%$ / $+15\%$ & -- \\
\midrule
$\alpha,\gamma\in\{2,2.5\}$, v19 & outside & 1.61 / 1.10 & 0.62 / 0.43 & 62\% / 61\% & -- \\
$\delta=4\sigma$, v19 & outside & 0.97 / 0.65 & 0.64 / 0.54 & 34\% / 17\% & -- \\
$\beta=-1$, v19 & sign & 0.38 / 0.39 & 0.41 / 0.26 & $-6\%$ / 34\% & 0.63 $\to$ 0.46 ($\bobs$ $-0.42\to-0.55$, OLS $-0.58$; $\bint$ $-0.80\to-0.95$) \\
$L=2048$, v19 & -- & 0.85 / 0.39 & 0.39 / 0.08 & 54\% / 79\% & -- \\
\bottomrule
\end{tabular}}
\end{table}

\begin{table}[H]
\caption{\textbf{E1 in full}: impact and cumulative coefficients.}
\label{tab:e1full}
\vskip 0.1in
\centering\small
\setlength{\tabcolsep}{4pt}
\resizebox{\textwidth}{!}{%
\begin{tabular}{llccccccc}
\toprule
Model & Vis. & impact slope ($R^2$) & cumul.\ slope ($R^2$) & mean $\CG$ & $\bint$ & $S_{\mathrm{int}}$ & $S_{\mathrm{obs}}$ & cells \\
\midrule
In-context OLS & \Vz & 1.00 (1.00) & $-0.10$ (0.05) & 0.41 & 0.92 & 1.27 & 1.36 & 25 \\
 & \Vo & 1.00 (1.00) & 1.00 (1.00) & 0.33 & 1.00 & 1.00 & 1.32 & 25 \\
 & \Vt & -- & 0.88 (0.70) & 0.00 & 1.00 & 1.00 & 1.00 & 25 \\
Chronos-2 & \Vz & 1.64 (0.99) & 0.98 (0.64) & 0.77 & 0.84 & 0.89 & 1.31 & 25 \\
 & \Vo & 1.78 (0.83) & 1.58 (0.94) & 0.73 & 0.81 & 0.88 & 1.31 & 25 \\
 & \Vt & -- & -- & 0.40 & 0.97 & 0.97 & 1.09 & 25 \\
TiRex-2 & \Vz & 1.34 (1.00) & 0.79 (0.73) & 0.62 & 0.95 & 0.87 & 1.30 & 25 \\
 & \Vo & 1.56 (0.83) & 1.19 (0.98) & 0.62 & 0.97 & 0.86 & 1.30 & 25 \\
 & \Vt & -- & -- & 0.77 & 0.63 & 0.65 & 1.22 & 25 \\
TabPFN-TS & \Vz & 1.00 (0.98) & 0.76 (0.97) & 0.57 & 0.88 & 0.55 & 1.08 & 9 \\
 & \Vt & -- & -- & 0.50 & 0.89 & 0.56 & 1.09 & 9 \\
TimesFM-2.5 + XReg & \Vz & 1.26 (0.90) & 0.64 (0.90) & 0.95 & 1.32 & 0.67 & 1.15 & 9 \\
 & \Vo & 1.62 (0.96) & 0.82 (0.96) & 0.95 & 1.32 & 0.67 & 1.15 & 9 \\
 & \Vt & -- & -- & 0.78 & 1.32 & 0.67 & 1.15 & 9 \\
TimesFM-2.5 (univariate) & all & -- & -- & 0.00 & 0.00 & 0.00 & 0.00 & 9 \\
\bottomrule
\end{tabular}}
\end{table}

\begin{table}[H]
\caption{\textbf{E2 null effect} ($\beta=0$; mean spurious impact coefficient over nine cells; false-positive rate = cells whose 95\% CI excludes 0). Randomized-context false positives with $|\text{effect}|\le0.04$ are trivially significant at $n=600$ \citep{wasserstein2016asa}; the magnitude column is the informative one.}
\label{tab:null}
\vskip 0.1in
\centering\footnotesize
\setlength{\tabcolsep}{4pt}
\begin{tabular}{llccccc}
\toprule
Model & Vis. & observational & least-squares pred. & randomized & FP obs. & FP rand. \\
\midrule
Chronos-2 & \Vo & 0.47 & 0.42 & 0.00 & 100\% & 11\% \\
Chronos-2 & \Vt & 0.16 & 0.00 & 0.00 & 100\% & 56\% \\
TiRex-2 & \Vo & 0.48 & 0.42 & 0.00 & 100\% & 0\% \\
TiRex-2 & \Vt & 0.31 & 0.00 & $-0.01$ & 100\% & 78\% \\
TabPFN-TS & \Vt & 0.44 & 0.00 & 0.00 & 100\% & 11\% \\
TimesFM-2.5 + XReg & \Vo & 0.72 & 0.42 & $-0.01$ & -- & -- \\
TimesFM-2.5 + XReg & \Vt & 0.31 & 0.00 & 0.00 & -- & -- \\
Chronos-2 + \IPT\ v1 ($\beta>0$ prior) & \Vo & 0.31 & 0.42 & \textbf{0.24} & 100\% & 100\% \\
Chronos-2 + \IPT\ v2 & \Vo & 0.28 & 0.42 & 0.00 & 100\% & 89\% \\
Chronos-2 + \IPT\ v4 & \Vo & 0.25 & 0.42 & 0.04 & 100\% & 100\% \\
Chronos-2 + \IPT\ v4 & \Vt & 0.19 & 0.00 & $-0.01$ & 100\% & 89\% \\
Chronos-2 + \IPT\ v5 & \Vo & 0.35 & 0.42 & 0.12 & 100\% & 89\% \\
Chronos-2 + v10 (do-flag) & \Vo & 0.42 & 0.42 & 0.00 & 100\% & 100\% \\
Chronos-2 + \IPT\ v11 (KernelSynth replay) & \Vo\ / \Vt & 0.32 / 0.21 & 0.42 / 0.00 & 0.06 / 0.03 & -- & -- \\
Chronos-2 + \IPT\ v13 (LoRA + replay) & \Vo\ / \Vt & 0.43 / 0.22 & 0.42 / 0.00 & 0.05 / 0.02 & -- & -- \\
Chronos-2 + \IPT\ v19 (LoRA + replay 4{,}500) & \Vo\ / \Vt & 0.41 / 0.23 & 0.42 / 0.00 & 0.08 / 0.01 & -- & -- \\
\bottomrule
\end{tabular}
\end{table}

\begin{figure}[H]
\centering
\includegraphics[width=0.48\textwidth]{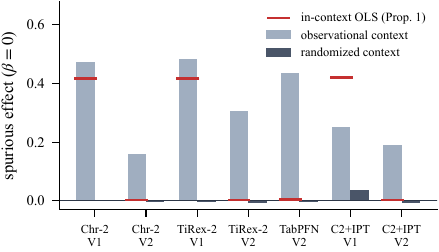}\hfill
\includegraphics[width=0.48\textwidth]{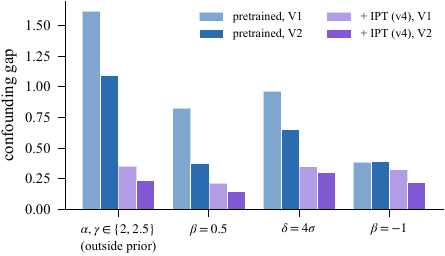}
\caption{Left: E2 spurious effects. Right: E8 out-of-prior transfer of post-training (confounding gap before and after \IPT).}
\end{figure}

\begin{table}[H]
\caption{\textbf{E3 anticipation leak} (onset $H/2$; ratio of pre-onset to post-onset response energy; the preregistered claim requires a CI lower bound $>0.05$). No systematic leak; heavy right tail under randomized context.}
\vskip 0.1in
\centering\small
\begin{tabular}{llccc}
\toprule
Model & Regime & \Vz\ median (97.5\%) & \Vo\ median (97.5\%) & \Vt\ median (97.5\%) \\
\midrule
Chronos-2 & obs. & 0.002 (0.006) & 0.001 (0.006) & 0.003 (0.030) \\
Chronos-2 & rand. & 0.011 (0.331) & 0.012 (0.495) & 0.013 (0.065) \\
TiRex-2 & obs. & 0.001 (0.004) & 0.001 (0.004) & 0.001 (0.008) \\
TiRex-2 & rand. & 0.008 (0.071) & 0.008 (0.068) & 0.035 (1.796) \\
\bottomrule
\end{tabular}
\end{table}

\begin{table}[H]
\caption{\textbf{E4 DoTime dot-Identifiability-v1} (12{,}150 episodes; effect-direction accuracy, NMSE, and direction accuracy on key structures).}
\label{tab:e4}
\vskip 0.1in
\centering\small
\setlength{\tabcolsep}{3.5pt}
\resizebox{\textwidth}{!}{%
\begin{tabular}{lcccccc}
\toprule
Method & dir.\ acc. & NMSE & back-door & front-door & IV & bow \\
\midrule
Zero / Mean / AR(1) & 0.53 / 0.53 / 0.53 & 1.03 / 0.68 / 0.95 & 0.50--0.53 & 0.51--0.53 & 0.51--0.54 & 0.54--0.56 \\
Naive OLS (no adjustment) & \textbf{0.77} & 0.57 & 0.79 & 0.70 & 0.81 & 0.81 \\
Back-door OLS & 0.60 & 0.63 & \textbf{0.81} & 0.53 & 0.54 & 0.54 \\
VAR-OLS / IV-2SLS & 0.53 / 0.55 & 0.82 / 0.67 & 0.52 / 0.53 & 0.54 / 0.53 & 0.52 / 0.61 & 0.54 / 0.54 \\
Oracle & 1.00 & 0.00 & 1.00 & 1.00 & 1.00 & 1.00 \\
Chronos-2 & 0.684 & \textbf{0.44} & 0.68 & 0.60 & 0.74 & 0.73 \\
Chronos-2 + \IPT\ v2 / v3 / v4 & 0.69 / 0.70 / 0.71 & 0.44 / 0.46 / 0.45 & 0.72 / 0.71 / 0.72 & 0.62 / 0.61 / 0.63 & 0.73 / 0.73 / 0.74 & 0.74 / 0.75 / 0.76 \\
Chronos-2 + \IPT\ v19 & 0.710 & \textbf{0.43} & 0.73 & 0.61 & 0.76 & 0.76 \\
TiRex-2 & 0.674 & 0.49 & 0.67 & 0.58 & 0.69 & 0.70 \\
\bottomrule
\end{tabular}}
\end{table}

\begin{table}[H]
\caption{\textbf{E6 real systems in full.} Mean horizon-average effect (fraction of the identified effect), slope on the in-window predicted bias (1 associational, 0 adjusting), and RMSE against the identified effect. Light tunnel: the in-window least-squares reference is itself unstable (six unobserved LED/polariser settings and persistent random walks within 512 steps), so the slope reference is weaker than on the wind tunnel; the identified effect from the pooled regression is reliable ($R^2$ 0.976).}
\label{tab:real}
\vskip 0.1in
\centering\small
\setlength{\tabcolsep}{4pt}
\begin{tabular}{l ccc ccc ccc}
\toprule
& \multicolumn{3}{c}{Wind tunnel: effect / truth} & \multicolumn{3}{c}{slope on predicted bias} & \multicolumn{3}{c}{RMSE (truth 16.4)} \\
\cmidrule(lr){2-4}\cmidrule(lr){5-7}\cmidrule(lr){8-10}
Model & \Vz & \Vo & \Vt & \Vz & \Vo & \Vt & \Vz & \Vo & \Vt \\
\midrule
In-context OLS & 0.42 & 0.79 & 0.77 & 0.43 & \textbf{1.00} & \textbf{0.00} & 21.2 & 34.8 & \textbf{7.9} \\
Chronos-2 & 0.33 & 0.41 & 0.47 & 0.07 & 0.01 & $-0.03$ & 13.1 & 11.3 & 10.5 \\
TiRex-2 & 0.18 & 0.19 & 0.22 & 0.04 & 0.04 & 0.01 & 15.9 & 15.2 & 14.1 \\
Chronos-2 + \IPT & 0.65 & 0.63 & 0.73 & 0.41 & 0.22 & 0.03 & 17.9 & 13.6 & 8.3 \\
Chronos-2 + \IPT\ v19 & 0.46 & 0.42 & 0.56 & 0.14 & 0.02 & $-0.03$ & 13.1 & 12.2 & 10.1 \\
\midrule
& \multicolumn{3}{c}{Wind tunnel, roles swapped: effect / truth} & \multicolumn{3}{c}{slope on predicted bias} & \multicolumn{3}{c}{RMSE (truth $-20.1$)} \\
\cmidrule(lr){2-4}\cmidrule(lr){5-7}\cmidrule(lr){8-10}
In-context OLS & -- & 0.92 & 0.90 & -- & \textbf{0.91} & $\mathbf{-0.02}$ & -- & 13.6 & \textbf{8.5} \\
Chronos-2 & 0.65 & 0.72 & 0.77 & 0.17 & 0.06 & 0.00 & 10.7 & 9.7 & 9.1 \\
TiRex-2 & 0.44 & 0.43 & 0.39 & 0.12 & 0.13 & 0.05 & 14.3 & 14.5 & 15.2 \\
Chronos-2 + \IPT\ v19 & 0.69 & 0.83 & 0.85 & 0.20 & 0.09 & 0.02 & 10.8 & 10.0 & 9.1 \\
\midrule
& \multicolumn{3}{c}{Light tunnel: effect / truth} & \multicolumn{3}{c}{slope on predicted bias} & \multicolumn{3}{c}{RMSE (truth 1.87)} \\
\cmidrule(lr){2-4}\cmidrule(lr){5-7}\cmidrule(lr){8-10}
In-context OLS & 1.34 & 1.37 & 1.24 & 0.41 & 0.40 & 0.11 & 5.3 & 5.2 & 6.6 \\
Chronos-2 & 0.56 & 0.74 & 0.84 & 0.04 & 0.04 & 0.00 & 2.2 & 1.8 & \textbf{1.6} \\
Chronos-2 + \IPT & 1.64 & 2.09 & 1.91 & 0.10 & 0.09 & $-0.10$ & 2.9 & 3.1 & 3.5 \\
Chronos-2 + \IPT\ v19 & 1.15 & 1.33 & 1.52 & 0.06 & 0.05 & $-0.05$ & 3.1 & 2.6 & 2.5 \\
Chronos-2 + \IPT\ v20 (wide scale rand.) & 1.16 & 1.19 & 1.19 & 0.02 & 0.03 & $-0.05$ & 2.6 & 2.1 & 2.0 \\
Chronos-2 + \IPT\ v22 ($X,Y$ scale rand.) & 1.06 & 1.28 & 1.41 & 0.05 & 0.06 & $-0.05$ & 3.1 & 2.5 & 2.3 \\
\bottomrule
\end{tabular}
\end{table}

\end{document}